\documentclass[10pt,journal,compsoc]{IEEEtran}

\usepackage[utf8]{inputenc} 
\usepackage[T1]{fontenc}    
\usepackage{url}            
\usepackage{booktabs}       
\usepackage{amsthm,amssymb,amsmath,amsfonts}    
\usepackage{mathtools}
\usepackage{nicefrac}       
\usepackage{microtype}      
\usepackage{pifont}

\usepackage{bm}
\usepackage{mathtools}
\usepackage[ruled, lined, linesnumbered, commentsnumbered, longend]{algorithm2e}
\usepackage{multirow}
\usepackage{physics}
\usepackage{bbding} 
\usepackage{array}

\usepackage{caption}
\usepackage{fancyhdr}
\usepackage{amssymb}
\usepackage{pifont}
\usepackage{xr}
\usepackage{wrapfig}
\usepackage{subfigure}
\usepackage{setspace}
\usepackage{exscale}
\usepackage{relsize}
\usepackage{ragged2e}
\usepackage{cite} 
\usepackage[table,xcdraw]{xcolor}
\definecolor{oxfordblue}{rgb}{0.0, 0.13, 0.28} 
\definecolor{carnelian}{rgb}{0.7, 0.11, 0.11}  
\definecolor{darkslategray}{rgb}{0.18, 0.31, 0.31} 
\definecolor{skyblue}{rgb}{0.53, 0.81, 0.92}    
\definecolor{cadetblue}{rgb}{0.37, 0.62, 0.63} 
\definecolor{carnelian}{rgb}{0.7, 0.11, 0.11}  

\usepackage[
    colorlinks=true,
    linkcolor=carnelian,        
    citecolor=carnelian,        
    urlcolor=skyblue,     
    bookmarks=true,
    bookmarksopen=true,
    bookmarksnumbered=true,
    bookmarkstype=toc
]{hyperref}

\newcolumntype{P}[1]{>{\centering\arraybackslash}p{#1}}
\newtheorem{theorem}{Theorem}
\newtheorem{lemma}{Lemma}

\newtheorem{assumption}{Assumption}

\DeclareMathOperator*{\argmin}{argmin}

\def\*#1{\mathbf{#1}}
\def\$#1{\mathcal{#1}}
\def\^#1{\mathbb{#1}}

\def\innerprod#1{\left\langle#1\right\rangle}

\definecolor{carnelian}{rgb}{0.7, 0.11, 0.11}  

\newcommand{\comment}[1]{}

\newcommand{\method}[1]{LoCo} 

\newcommand{\emti}[2]{\tilde{\*{e}}_{#1}^{#2}}

\newcommand{\gmbi}[1]{\bar{\*{g}}_{#1}}
\newcommand{\emi}[2]{\*{e}_{#1}^{#2}}
\newcommand{\hmi}[2]{\*{h}_{#1}^{#2}}
\newcommand{\gmi}[2]{\*{g}_{#1}^{#2}}
\newcommand{\mmi}[1]{\*{m}_{#1}} 
\newcommand{\vmi}[1]{\*{v}_{#1}}
\newcommand{\xmi}[1]{\bm{\theta}_{#1}} 
\newcommand{\EE}{{\mathbb{E}}}
\newcommand{\ginf}{c_{\infty}}

\newcommand{\betai}[1]{\beta_{#1}} 
\newcommand{\deltami}[2]{\bm{\delta}_{#1}^{#2}} 
\newcommand{\emhi}[2]{\hat{\*{e}}_{#1}^{#2}} 
 
\newcommand{\led}[1]{\overset{\text{\ding{#1}}}{\leq}}

\newcommand{\lee}[1]{\overset{\text{\ding{#1}}}{=}}

\makeatletter
\newcommand*{\addFileDependency}[1]{
  \typeout{(#1)}
  \@addtofilelist{#1}
  \IfFileExists{#1}{}{\typeout{No file #1.}}
}
\newcommand{\algorithmfootnote}[2][\footnotesize]{%
  \let\old@algocf@finish\@algocf@finish
  \def\@algocf@finish{\old@algocf@finish
    \leavevmode\rlap{\begin{minipage}{\linewidth}
    #1#2
    \end{minipage}}%
  }%
}
\makeatother

\begin{document}
\title{\method{}: Low-Bit Communication Adaptor for Large-scale Model Training}

\author{Xingyu Xie, Zhijie Lin, Kim-Chuan Toh, Pan Zhou
\IEEEcompsocitemizethanks{\IEEEcompsocthanksitem X. Xie and KC. Toh are with the Department of Mathematics, National University of Singapore, Singapore.

\IEEEcompsocthanksitem P. Zhou is with the School of Computing and Information Systems, Singapore Management University, Singapore, and is the corresponding author {(panzhou@smu.edu.sg)}.

\IEEEcompsocthanksitem P. Zhou and Z. Lin were previously with Sea AI Lab, Singapore.

}

}

\markboth{}%
{Shell \MakeLowercase{\textit{et al.}}: Bare Demo of IEEEtran.cls for Computer Society Journals}

\IEEEtitleabstractindextext{%
\begin{abstract}
To efficiently train large-scale models, low-bit gradient communication compresses full-precision gradients on local GPU nodes into low-precision ones for higher gradient synchronization efficiency among GPU nodes. However, it often degrades training quality due to compression information loss. To address this, we propose the Low-bit Communication Adaptor (\method{}), which compensates gradients on local GPU nodes before compression, ensuring efficient synchronization without compromising training quality. Specifically, \method{} designs a moving average of historical compensation errors to stably estimate concurrent compression error and then adopts it to compensate for the concurrent gradient compression, yielding a less lossless compression. This mechanism allows it to be compatible with general optimizers like Adam and sharding strategies like FSDP. Theoretical analysis shows that integrating \method{} into full-precision optimizers like Adam and SGD does not impair their convergence speed on nonconvex problems. Experimental results  show that across large-scale model training frameworks like Megatron-LM and PyTorch's FSDP,  \method{} significantly improves communication efficiency, e.g., 
improving Adam's training speed by 14\% to 40\%  without performance degradation on large language models like LLAMAs and MoEs.
\end{abstract}

}

\maketitle

%
\IEEEpeerreviewmaketitle

\section{Introduction}\label{introduction}
\IEEEPARstart{D}{eep} learning has made remarkable strides across various domains in recent decades, such as language modeling~\cite{brown2020language,jiang2023mistral}, computer vision~\cite{bai2023sequential}, and multi-modality~\cite{liu2023visual}. This progress is largely attributed to the advent of large-scale models, like the GPT and LLAMA series~\cite{radford2019language,brown2020language,touvron2023llama,touvron2023llama2}, characterized by their billions of parameters and trillions of training tokens. This trend of large-scale models has expanded into various other fields, including finance~\cite{wu2023bloomberggpt}, law~\cite{sun2023short}, and medicine~\cite{zhou2023survey}. Despite their successes, these large-scale models necessitate extensive GPUs for parallel training, employing strategies like data parallelism~\cite{dean2012large}, pipeline parallelism~\cite{harlap2018pipedream}, tensor parallelism~\cite{shoeybi2019megatron}.  
A major challenge in this parallel training is the frequent gradient communication for synchronization among GPUs, which significantly burdens the communication system. In fact, the communication time can even consume over 50\% of the total training time in some cases~\cite{tang20211,lu2022maximizing}.

To relieve the communication burden, one often adopts compression techniques, e.g., quantization, to compress the full-precision communication variables into low-precision formats, e.g., 32-bit  gradient to 8-bit one. While significantly improving communication efficiency among GPU nodes, this compression often brings substantial challenges in maintaining training quality due to information loss.  Notably, low-bit gradients (e.g., less than 8-bit) are not typically supported by hardware, and their computations, e.g., addition, often suffer from overflow. Consequently, advanced frameworks like Megatron-LM~\cite{shoeybi2019megatron} and Pytorch Fully Sharded Data Parallelism (FSDP)~\cite{zhao2023pytorch}, though capable of low-bit computations for weights and activations, still need high-precision gradient for communication to avoid performance degradation that is particularly pronounced in models like large language models (LLMs). This well testifies the significant challenges in gradient compression for training large-scale models.

{  To address the challenge of communication efficiency in large-scale model training, error-feedback compression (EFC)~\cite{seide20141,richtarik2021ef21} has been developed to compensate for communication variables before compression, ensuring small compression errors.  This technique has been effectively applied in gradient compression, enabling the development of communication-efficient low-bit optimizers, such as 1-bit Adam~\cite{tang20211} and 1-bit LAMB~\cite{li20221}. Despite their success, many EFC methods still face critical challenges that limit their scalability in distributed training environments.
    
Firstly, a majority of EFC methods~\cite{tang2019doublesqueeze,richtarik2021ef21,gruntkowska2023ef21} are inherently designed for parameter-server architectures, where a central server coordinates communication among all devices. While effective for smaller-scale systems, this architecture often struggles with scalability due to its reliance on centralized communication. In contrast, collective communication architectures, which distribute communication workloads across devices, often organized in ring or tree topologies, offer improved scalability and efficiency~\cite{patarasuk2009bandwidth}, making them the preferred choice for modern LLM training.  
However, adapting EFC methods to collective communication architectures is not straightforward. It requires not only the redesign of communication protocols but also the implementation of collective communication primitives. Additionally, EFC methods often rely on maintaining global variables for error compensation, which introduces significant communication and memory overheads. For instance, 1-bit LAMB requires each node to store and synchronize a service error variable proportional to the model size, resulting in memory overhead that can become a bottleneck in large-scale model training. 
More importantly, in the context of FSDP, where model parameters and optimizer states are partitioned and distributed across devices, many EFC methods encounter fundamental compatibility issues. For example, 1-bit Adam, which relies on compressing and synchronizing optimizer states, faces challenges in FSDP because complete optimizer states are not maintained during backpropagation, making it difficult to update model weights without incurring additional communication overhead. Similarly, methods like IntSGD~\cite{mishchenko2021intsgd} require access to full model parameters for subsequent computations, resulting in higher communication costs in sharded setups.
}

\noindent\textbf{Contribution.} In this work, we propose an effective gradient compression approach to improve the communication efficiency of widely used optimizers, particularly in FSDP settings. We introduce the novel Low-bit Communication Adaptor (\method{}), which compresses full-precision gradients into low-precision ones with minimal information loss, significantly enhancing communication efficiency in large-scale training. The core of \method{} lies in its refined error-feedback mechanism: it leverages a moving average of past compression errors to compute an error term, which is combined with the gradient before compression to reduce overall compression error. This strategy decouples \method{} from specific optimization algorithms, making it compatible with various optimizers, such as Adam~\cite{kingma2014adam} and AdaFactor~\cite{shazeer2018adafactor}, and integrates seamlessly with modern sharding strategies like FSDP and Zero~\cite{rajbhandari2020zero}.
Furthermore, \method{} has been adapted to popular large-scale model training frameworks such as Megatron-LM, PyTorch's FSDP, and DeepSpeed~\cite{deepspeed}. This compatibility ensures that \method{} can be effectively employed in diverse large-model training environments, providing robust support for scaling up training processes.

An important aspect of our work is the theoretical analysis, demonstrating that for nonconvex problems—including the training of large-scale models—integrating \method{} with standard optimizers like SGD and Adam-type algorithms does not adversely affect their convergence speed. This ensures that these optimizers retain their effectiveness even when operating with low-precision gradients.

Extensive experiments show that \method{} significantly improves efficiency while maintaining performance comparable to full-precision optimizers. For instance, on various LLMs, such as LLAMAs and MoE-Mixtral~\cite{jiang2024mixtral} with model sizes ranging from 7B to 70B parameters, \method{} enhances the overall training speed of Adam by 14\% to 40\% while preserving comparable downstream task performance as Adam with full-precision gradients. This demonstrates the effectiveness of \method{} in alleviating communication burdens in large-scale model training.

\vspace{-2mm}
\section{Preliminary and Related Work}\label{sec:relatedwork}
\subsection{Communication among GPU Nodes}\label{sec:RC}
{ In distributed training, communication architectures are broadly categorized into parameter-server architectures and decentralized collective communication architectures.

Parameter-server architectures adopt a centralized design where a parameter server aggregates gradients from worker nodes and broadcasts updated parameters back. While simplifying coordination, it often suffers from scalability issues, as the parameter server can become a communication bottleneck under heavy workloads. Moreover, the central server represents a single point of failure, meaning that any disruption at the server can completely halt the training.

In contrast, decentralized collective communication architectures distribute the communication workload across all nodes, enhancing scalability and fault tolerance. These architectures typically employ network topologies such as ring and tree, along with communication primitives like reduce-scatter and all-gather for efficient data exchange.
Ring topology~\cite{patarasuk2009bandwidth} connects nodes in a closed loop, enabling simultaneous data exchange between neighboring nodes. This setup balances communication loads and achieves high parallelism. However, it may introduce latency in large-scale systems, as data must traverse multiple rounds to reach all nodes.  
Tree topology~\cite{huang2021communication} organizes nodes hierarchically, enabling efficient gradient aggregation and parameter distribution. Compared to ring topology, tree-based communication reduces latency for large-scale transfers but can lead to imbalanced workloads near the root, requiring careful optimization.
By addressing centralized bottlenecks and leveraging collective communication primitives, decentralized architectures provide greater scalability and resilience, making them a preferred choice for large-scale model training.}

\vspace{-2mm}
\subsection{Fully Sharded Data Parallelism}
FSDP has emerged as the preferred training method for large-scale machine learning models, addressing limitations that make Distributed Data Parallel (DDP) unsuitable for such tasks. DDP, which requires each GPU to maintain a full replica of the model, faces significant memory constraints when dealing with models that have billions of parameters. In contrast, FSDP improves scalability by sharding model parameters, gradients, and optimizer states across multiple devices. This sharding process allows FSDP to dynamically gather only the necessary shards for computation, thus substantially reducing memory usage and enabling efficient training of extremely large models. Integrated into frameworks like PyTorch and Megatron-LM~\cite{shoeybi2019megatron,narayanan2021efficient}, FSDP has shown considerable improvements in training speed and memory efficiency, solidifying its role as the default solution for large model training~\cite{zhao2023pytorch,rajbhandari2020zero}. For a comprehensive background and discussion, please refer to Appendix Sec.~\ref{sec:fsdp}.

 \vspace{-2mm}
\subsection{Communication-efficient Training} 
Recently,   AI models have become much larger than before, like billion-scale language models and multi-modal models~\cite{touvron2023llama,jiang2024mixtral}, and their training bottleneck is often the high communication cost caused by the very high-dimensional gradient communication among   GPUs. To alleviate this issue, one often compresses the gradient before its communication.  Currently,  compression techniques mainly contain gradient  quantization~\cite{seide20141,alistarh2017qsgd,wen2017terngrad}, gradient spasification~\cite{wangni2018gradient,wang2018atomo,shi2021towards}, and decentralization~\cite{lian2017can,lu2021optimal}.
 Among them, gradient quantization aims to quantize the high precision gradient into a low-bit one for reducing communication cost,  and has shown promising efficiency for model training, e.g., 1-bit Adam~\cite{tang20211} and 0/1 Adam~\cite{lu2022maximizing} of which both compress the entries in the gradient-based statistics into $\pm 1$. 
 
 \vspace{-2mm} 
\subsection{Error-feedback Compression} Gradient compression often introduces information loss, leading to accumulated errors that can cause algorithmic divergence. To address this, Seide et al.~\cite{seide20141} proposed the first error-feedback compression (EFC) strategy, which compensates for compression errors by adding them back into the gradient before compression. This method demonstrated effectiveness in 1-bit SGD. After this, EF21 was proposed~\cite{richtarik2021ef21}, a theoretically and practically improved EFC variant, which has inspired further theoretical developments~\cite{gruntkowska2023ef21,zhao2022beer,li2020acceleration,li2020unified}.
Practical adaptations of EFC have also been explored, incorporating gradient quantization into adaptive gradient algorithms to develop communication-efficient variants, such as 1-bit Adam~\cite{tang20211} and 0/1 Adam~\cite{lu2022maximizing}. Adaptive gradient algorithms like AdaGrad~\cite{duchi2011adaptive}, Adam~\cite{kingma2014adam}, and Adan~\cite{xie2022adan} adjust learning rates for each gradient coordinate based on the curvature of the training loss, offering faster convergence than SGD. Combining EFC with these adaptive algorithms enhances training efficiency while maintaining comparable performance to their uncompressed counterparts.

Despite the promising results of EFC, most implementations are tailored for parameter-server architecture, limiting their applicability. Some other EFC-based methods, such as PowerSGD~\cite{vogels2019powersgd}, can train neural networks with DDP patterns, but they encounter challenges in the FSDP setting. For instance, in PyTorch's FSDP framework, gradients retrieved during communication are flattened, complicating the application of matrix decomposition-based techniques like PowerSGD, which rely on the original gradient shapes.

\begin{figure}[t]
    \centering		
  \includegraphics[width=\linewidth]{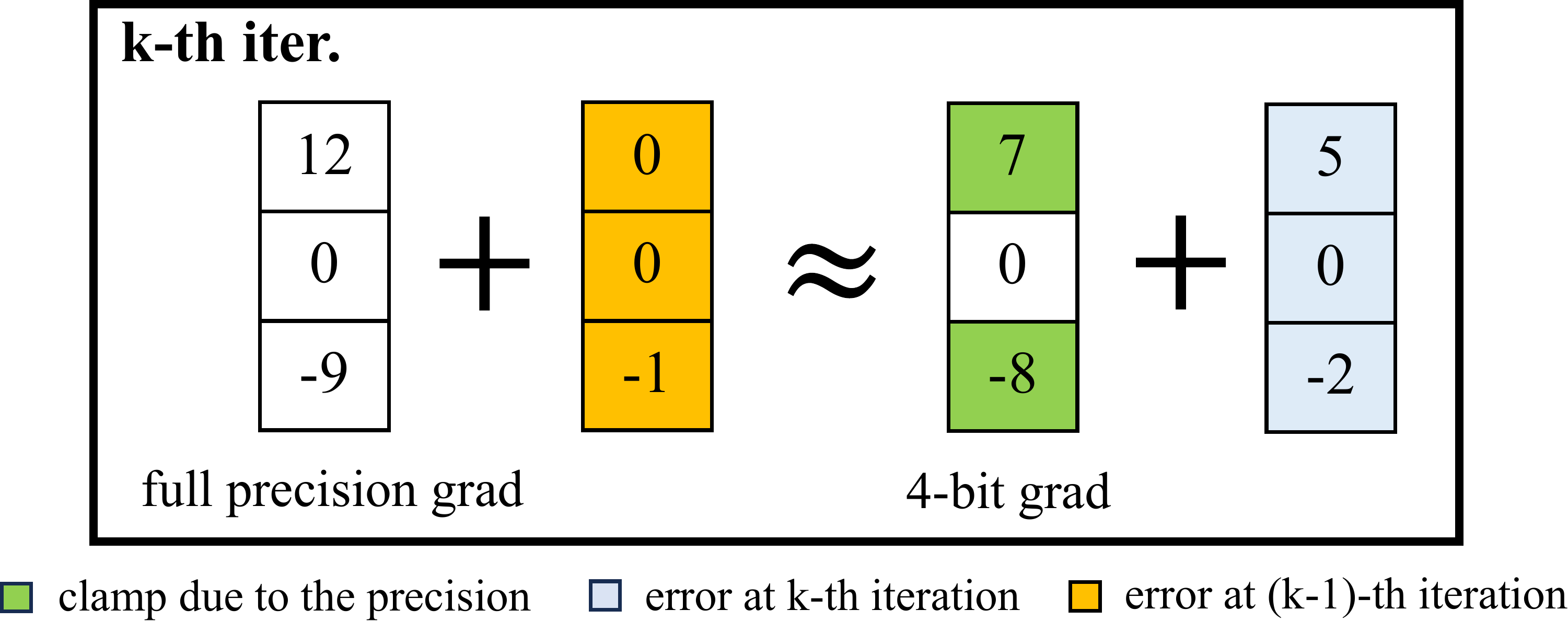}	
  \caption{Illustration of \method{}. At iteration $k$, before compression,\method{} compensates the full-precision gradient with the compression error from previous iterations to reduce the compression error. Then, it compresses the gradient into a low-bit one for fast communication among GPU nodes.
}\label{fig:intro}
\vspace{-5mm}
\end{figure}	

\vspace{-2mm}
\subsection{Challenges in Migrating EFC to FSDP}\label{sec:challenge}
Migrating EFC methods from parameter-server architectures to FSDP frameworks necessitates addressing significant challenges associated with maintaining the global error variable. In the original EFC algorithms, a global error variable compensates for compression errors. Removing this global variable is crucial to reduce memory and communication overheads, as maintaining a copy of the global error variable on each node would be inefficient.

\textbf{Sharding Conflict:} Specifically, for optimizers like 1-bit Adam and 0/1 Adam, which use optimizer state communication instead of gradient communication, transitioning to FSDP is particularly challenging. In the FSDP context, gradients are complete during the backward pass and can be effectively reduced and scattered. However, using optimizer states for model updates requires additional communication of these states, leading to increased overhead that contradicts the sharding strategy's goal of minimizing communication.

\textbf{Memory Constraints with Optimizer States:} Some methods, such as 0/1 Adam and EF21-SGD2M~\cite{fatkhullin2024momentum}, compress or communicate optimizer states, leading to significant memory management challenges. While gradients can be discarded after communication to free up memory, optimizer states, particularly first-order moments, must be retained for subsequent update iterations. These persistent states impose significant memory demands in sharded environments. Furthermore, optimizer states are not suitable for low-precision representation. For instance, FP8-LM~\cite{peng2023fp8} demonstrates that the precision of optimizer states significantly impacts model training quality. Detailed experiments from FP8-LM show that optimizer states require precision higher than 16-bit for effective training. In contrast, our \method{} maintains only a local average of the compressed errors, which is less sensitive to precision and can be stored in 8-bit format without impacting training effectiveness.

\textbf{Matrix Decomposition Compression Challenges:} Additionally, methods based on matrix decomposition, such as PowerSGD~\cite{vogels2019powersgd}, face difficulties in the FSDP context. In PyTorch's FSDP framework, gradients retrieved in the communication hook are flattened, complicating the application of matrix decomposition techniques that rely on the gradients' original shapes. This restriction prevents the straightforward application of such compression methods, further highlighting the challenges of adapting EFC to FSDP.

\begin{algorithm}[t]
    \SetAlgoLined
    \KwIn{initialization $\bm{\theta}_0$, compression scalar $s$ and $s_e$, reset frequency $T_{c}$,    $\beta \in [0,1]$.}
    \KwOut{model weight  $\bm{\theta}_{K}$ at the $K$-th iteration.}
  \SetKwBlock{Stepone}{Step 1. Low-Bit Gradient  Estimation}{}
  \SetKwBlock{Steptwo}{Step 2. Compensation Error Estimation}{}
  \SetKwBlock{Stepthree}{Step 3. Communication and Model Update}{}
		\While{$k<K$}{
			\Stepone{
			estimate the gradient  ${\*g}^n_{k}$ at $\bm{\theta}_{k}$ on \textbf{Node} $n$\;
			$\*h^n_{k+1}= {\*g}^n_{k} + \operatorname{decompressor}(\*e^n_k; s_e)$\;
			
			$\tilde{\*h}^n_{k+1} =\operatorname{compressor}(\*h^n_{k+1}; s, 4)$\;
   }
			\Steptwo{
			${\*d}^n_{k+1} =\operatorname{decompressor}(\tilde{\*h}^n_{k+1}; s)$\;
			$\tilde{\*e}^n_{k+1}  =(1-\beta)\tilde{\*e}^n_{k} + \beta (\*h^n_{k+1} - {\*d}^n_{k+1})$\;
   \eIf{if $k \% T_c = 0$}{
    $\*e^n_{k+1} = \bm{0}$ \tcp*{error reset}
}{
    $\*e^n_{k+1} =\operatorname{compressor}(\tilde{\*e}^n_{k+1}; s_e, 8)$\; 
}
   }
			\Stepthree{
			$\tilde{\*g}_{k} = \frac{1}{N}\sum_{n=1}^N \operatorname{decompressor}(\tilde{\*h}^n_{k+1})$\;
   use $\tilde{\*g}_{k}$ to update  $\bm{\theta}_{k}$ on \textbf{Node} $n$\;
		}}
\caption{\textbf{\method{}} {\fontsize{8.5}{5}\selectfont{(Low-bit Communication Adaptor)}}
}\label{alg:BitCom}
\end{algorithm}

\vspace{-2mm}
\section{Low-Bit Communication Adaptor}\label{sec:methods}

To address the communication burden in large-model training on many GPU nodes, we introduce an efficient and novel low-bit communication adapter, \method{}.  As shown in Fig.~\ref{fig:intro}, the core idea of \method{} is to quantize the full-precision gradient into a lower-precision one with error-feedback for improving communication efficiency,  e.g., compressing a 32-bit gradient into a 4-bit one.    
A critical challenge in gradient compression is that its compression error accumulates along training iterations and can lead to failure in model training.  To solve this issue, \method{} employs a novel error-feedback compression strategy in Algorithm~\ref{alg:BitCom} to reduce accumulated errors and ensure high-quality compression. The strategy encompasses three key steps: 1) low-bit gradient estimation, 2) compensation error estimation, and 3) gradient communication and model update. Initially, \method{} computes the stochastic gradient on each GPU and compensates it before compression to a low-bit format. The compensation error estimation step integrates the current and historical compensation errors to mitigate gradient compression error accumulation in subsequent steps. The final step involves aggregating the average of low-bit gradients across GPUs, which is subsequently followed by the model update using optimizers like Adam and Adafactor. 
Detailed elaboration on these steps will be provided below.

\subsection{Low-bit Gradient Estimation}
As shown in Algorithm~\ref{alg:BitCom}, for parallel training with $N$ GPU nodes, each node receives a minibatch of data at each iteration to compute the minibatch stochastic gradient ${\*g}^n_{k}$. The key challenge is to compress the high-precision gradient ${\*g}^n_{k}$ into a low-precision form without causing significant accumulated compression errors in each iteration. This ensures that low-bit gradients can be efficiently transferred among GPUs while maintaining the training quality.  

Accordingly, we design an efficient and effective error-feedback-based compression strategy.  Specifically, we first define  the element-wise compression operation and its inverse operation as follows:
\begin{equation}\label{eq:compressor}
\begin{cases}
\operatorname{compressor}(\*h;  s, p) \coloneqq
    \operatorname{round}_{\operatorname{p-bit}}\qty(\*h\times s), \\
    \operatorname{decompressor}(\tilde{\*h}; s) \coloneqq \operatorname{float}(\tilde{\*h})/s, 
\end{cases}     
\end{equation}
where $s>0$ denotes a hyper-parameter that modulates the low-bit scale, and the rounding function $\operatorname{Round}_{\operatorname{p-bit}}$ is defined to round each floating-point number to its nearest integer within the range of $-2^{p-1}$ to $2^{p-1}-1$, such as  $-8$ to $+7$ in $\operatorname{Round}_{\operatorname{4-bit}}$ operation. The operation in Eqn.~\eqref{eq:compressor} is a commonly used quantization method. While more advanced methods like  IntSGD~\cite{mishchenko2021intsgd} exist, we found that the default one in PyTorch, Eqn.~\eqref{eq:compressor}, is sufficient for our tasks.

Given the current compensation error $\*e^n_k$,   we first add it back to the current stochastic gradient ${\*g}^n_{k}$ on each node:
\begin{equation}\label{eq:adad}
	\begin{split}
\*h^n_{k+1}= \*g^n_{k} + \operatorname{decompressor}(\*e^n_k; s_e).
	\end{split}     
\end{equation}
In \method{}, to save memory on each GPU, we use an 8-bit compensation error $\*e^n_k$ which is quantized by the operation $\operatorname{compressor}(\*h;  s_e, 8) $ with scale $s_e$ in Eqn.~\eqref{eq:compressor}. Here $\operatorname{decompressor}(\*e^n_k; s_e)$ inversely transfers 8-bit $\*e^n_k$ into its float version  for addition with $\*g^n_{k}$.  In Sec.~\ref{secerror}, we will introduce the compensation error $\*e^n_k$ which intuitively denotes the average of the previous gradient compression error.

Next, on each node,  we compress  the high-precision compensated gradient $	\tilde{\*h}^n_{k+1}$ into its low-bit version:
\begin{equation}\label{eq:comadadaspressor}
	\begin{split}
	\tilde{\*h}^n_{k+1} =\operatorname{compressor}(\*h^n_{k+1}; s, 4).
	\end{split}     
\end{equation}
In our experiments, we set  a scaling factor $s=2^{19}$ for fine-tuning and  $s \in \{2^{19}, 2^{17}\}$ for pre-training.
In practice, we find that 4-bit can balance communication efficiency and performance well. Indeed, Fig.~\ref{fig:curve} shows that 4-bit \method{} can achieve comparable training loss as 16-bit based Adam.

Finally, one can transfer the low-bit compensated gradient $\tilde{\*h}^n_{k+1}$ among GPU nodes to compute their average for the model update by using an optimizer. See details in Sec.~\ref{secmodel}. Compared to 32-bit gradients, using 4-bit gradients for communication significantly enhances communication efficiency, particularly in environments with numerous GPUs. As the number of GPUs or the model size increases, communication costs grow linearly. This increase in communication overhead can cause GPUs to idle, waiting for data transfer, leading to a notable degradation in overall computational efficiency.

\vspace{-3mm}
\subsection{Compensation Error Estimation}\label{secerror}
Compression inherently leads to information loss, which is a significant challenge in gradient compression. Our approach focuses on estimating a compensation error $\tilde{\*e}^n_k$ to mitigate error accumulation. The 8-bit compensation error $\*e^n_k$ used in Eqn.~\eqref{eq:adad} is a memory-efficient, compressed version of $\tilde{\*e}^n_k$ to conserve GPU memory. A straightforward solution is to set the compensation error as the gradient compression error from the previous iteration, as implemented in the EFC~\cite{seide20141}:
\begin{equation}\label{adasda}
\tilde{\*e}^n_k = \argmin_{\*e} \norm{\*e +  {\*d}^n_{k}  - \*h^n_{k} }^2 =  \*h^n_{k} - {\*d}^n_{k}, 
\end{equation} 
where ${\*d}^n_{k} =\operatorname{decompressor}(\tilde{\*h}^n_{k}; s)$ decompresses the low-bit compensated gradient $\tilde{\*h}^n_{k}$ into high-precision floating number for computation, and $\norm{\cdot}$ is the vector $\ell_2$ norm.

Unfortunately, we empirically find that this estimation is not stable. This is because compression operations, e.g., our rounding operation, top-k operation~\cite{alistarh2018convergence} and random sparsification~\cite{stich2018sparsified}, often inherently exhibit discontinuous properties, and bring abrupt fluctuations in compensation error. 
This means that $\tilde{\*e}^n_k$ suffers from large variance.  
As a result,   using $\tilde{\*e}^n_k$ to compensate the gradient in Eqn.~\eqref{eq:comadadaspressor} would also bring abrupt fluctuations into the low-bit gradient,   potentially failing the model training.  To relieve this issue, we regularize the current compensation error $\tilde{\*e}^n_k$ to be not too far from the previous one $\tilde{\*e}^n_{k-1}$. This can avoid big $\tilde{\*e}^n_k$ fluctuations and improve its smoothness. Formally, we have:
\begin{align}
	\tilde{\*e}_k^n = & \argmin_{\*e} \frac{\beta}{2}\norm{\*e+ {\*d}^n_{k}  - \*h^n_{k} }^2 + \frac{1-\beta}{2}\norm{\*e - \tilde{\*e}_{k-1}^n}^2 \notag\\
	= & \qty(1-\beta) \tilde{\*e}_{k-1}^n + \beta \qty( \*h^n_{k} - {\*d}^n_{k}), \label{dada}
\end{align}
where $0\leq\beta\leq 1$ is a parameter to balance the two terms, and $\tilde{\*e}^n_0$ is initialized as zero. One can observe that Eqn.~\eqref{dada} effectively averages all historical compression errors to estimate the compensation error, thus avoiding large fluctuations better than Eqn.~\eqref{adasda}. Some EFC-based methods, like EF21-SGD2M~\cite{fatkhullin2024momentum}, suggest that moving averages may have theoretical benefits. However, they apply the average to the gradient rather than the error, resembling a combination of EFC and adaptive optimizers rather than addressing the instability issues targeted by our method.

Moreover, when using SGD or Adam-type optimizers to update the model, the model weight at the $k$-th iteration can be formulated as $\bm{\theta}_k = \bm{\theta}_{0} -  \sum_{i=1}^k \bm{\eta}_i \circ \*g_i$, where  $\*g_i$ is the high-precision gradient. For SGD, its element-wise learning rate $\bm{\eta}_i$ is a constant $\eta$, while for Adam-type optimizers, e.g., Adam, it is the combination of a preconditioner and a constant learning rate $\eta$ (see Eqn.~\eqref{eq:sgd-adam} in Sec.~\ref{sec:convergence}), and is also of the order $\order{\eta}$. In this context, for \method{}, we can show that:
\begin{equation}\label{eq:acm_err}
  \norm{\sum_{i=1}^k \bm{\eta}_i \circ \*g_i - \sum_{i=1}^k \bm{\eta}_i \circ \tilde{\*g}_i} = \order{\eta \|\emti{k}{}\|{}} = \order{\eta},  
\end{equation}
where $\tilde{\*g}_i$ is the low-precision counterpart of gradient ${\*g}_i$.
This indicates that although the compression error $\norm{\*g_k-\tilde{\*g}_i}$ for a single iteration is of the order $\order{\| \emti{k}{}\|}=\order{\eta}$, the compression error over iterations will \emph{not} accumulate and remains at the  order of  $\order{\eta}$. 
This is because the difference between the accumulated high-precision gradients $\sum^k_{i=1}\bm{\eta}_i \circ\*g_i$ and low-precision gradients $\sum^k_{i=1}\bm{\eta}_i \circ\tilde{\*g}_i$ does not increase with the number of iterations. This demonstrates \method{}'s capability to mitigate error accumulation. For a formal illustration and analysis, please refer to Lemma~\ref{lem:sum_mk} in Appendix~\ref{sec:convergence}.

 To save GPU memory footprint, we further compress the high-precision compensation error $\tilde{\*e}_k^n$ into an 8-bit version ${\*e}_k^n$ on each node. This is important and necessary for large-model training settings since compensation error is of the same size as the model and brings  GPU memory overhead.  
 
Moreover, we periodically reset the compensation error. This is because, along with training iterations, the compensation errors in the early iterations are out of date and are not suitable for current compensation error estimation due to the ever-changing optimization process and landscape.  This is especially true for network training due to their fast-changing landscape during optimization.  In this way, we arrive at our final low-bit compensation error ${\*e}_k^n$:
\begin{align}\label{adadsa}
\*e^n_{k+1} \!=\! \begin{cases}
	\*0,  & \text{if $k\ \%\ T_c =  0$}, \\
	\operatorname{compressor}(\tilde{\*e}^n_{k+1}; s_e, 8), & \text{otherwise.}
\end{cases}
\end{align}
 We set $T_c \in \{512,1024\}$ for all experiments, which works well, and set $s_e = 4s$ or $6s$ where $s$ is used in Eqn.~\eqref{eq:comadadaspressor}.  In practice, on large-model training, e.g., LLAMA,  with this 8-bit quantization,  \method{} improves the token throughput by $10\% - 40\%+$ while only bringing a marginal memory overhead of less than 10\%. See results in Tables~\ref{tab:speedup} and~\ref{tab:mem}.

\vspace{-2mm}
\subsection{Communication and Model Update}\label{secmodel}
To synchronize the gradients for model update at the $k$-th iteration, we need to collect low-bit gradients $\{\tilde{\*h}^n_{k+1} \}_{n=1}^N$ from all $N$ GPU nodes, and compute their average 
\begin{equation}\label{adadaasda}
 	\tilde{\*g}_{k}  =   \frac{1}{N}\sum\nolimits_{n=1}^N \operatorname{decompressor}(\tilde{\*h}^n_{k+1}),
 \end{equation}
 where $\operatorname{decompressor}(\cdot)$ decompresses the 4-bit gradient $\tilde{\*h}^n_{k+1}$ into high-precision floating number, which is then used in the optimizer for the model update on each node. 

Considering the demands of large-scale model training, we adopt the FSDP strategy that is commonly used for training  LLMs~\cite{shoeybi2019megatron,rajbhandari2020zero}.
FSDP partitions the $d$-dimensional gradient  into $N$ blocks, e.g., $\tilde{\*h}^n_k \coloneqq [\tilde{\*h}^n_{k,1}, \cdots, \tilde{\*h}^n_{k, N}]$ on the $n$-th node.  Then each node only collects its corresponding portion from all other nodes, e.g., $\frac{1}{N}\sum_{i=1}^N\operatorname{decompressor}(\*h^i_{k,n})$ on the $n$-th node, to save memory and communication cost.

Under FSDP settings, which are widely used for training large-scale models, gradient averaging typically employs the reduce-scatter operation. However, to collect gradients on all GPU nodes, reduce-scatter requires each node to decompress, sum, and recompress the low-bit vectors. This complex process progressively increases information loss due to numerical anomalies such as overflow or underflow.
To address this issue, we adopt the all-to-all (all2all) strategy. In the all2all approach, the $n$-th node first gathers all low-bit gradient partitions 
$[\*h^1_{k,n}, \cdots, \*h^N_{k, n}]$, and then averages them in high precision locally. This method eliminates the intermediate steps in reduce-scatter, thus preventing information loss. Furthermore, all2all maintains computational and communication efficiency comparable to reduce-scatter. Consequently, we use all2all under FSDP settings.
\emph{For more details about all2all, see Appendix~\ref{sec:all2all}.}

\subsection{Discussion and Comparison} 
Compared with previous communication-efficient network training algorithms like Zero++~\cite{wang2023zero++}, 1-bit Adam~\cite{tang20211}, 1-bit LAMB~\cite{li20221} and 0/1 Adam~\cite{lu2022maximizing}, \method{} distinguishes itself from them through its low computational and memory demands, enabling effective low-bit gradient training in large-scale models. It incorporates algorithmic improvements that significantly reduce compression error accumulation over training iterations, advancing beyond methods that use compression without error-feedback, such as Zero++~\cite{wang2023zero++}.
A key advantage of \method{} is its high compatibility with various optimizers (e.g., Adam, AdaFactor), collective communication, and other components essential for large-model training, such as FSDP. This compatibility is achieved by decoupling its error-feedback strategy from these specific configurations. In contrast, other low-bit optimizers, such as 1-bit Adam, 1-bit LAMB, and 0/1 Adam, have error-feedback mechanisms specifically designed for certain optimizers, limiting their applicability to sharding strategies.
This versatility aligns with the critical integration required for training SoTA LLMs and allows \method{} to be seamlessly integrated with more advanced techniques in the future.

\vspace{-2mm}
\section{Convergence Guarantee}\label{sec:convergence}
In this section, we provide the convergence guarantee for the proposed \method{} when applied to two prevalent types of optimizers: SGD and Adam-family optimizers. We focus on the following nonconvex optimization problem:
\begin{equation*}
\min\nolimits_{\bm{\theta}}   f(\bm{\theta}) \coloneqq  \EE_{\bm{\zeta}\sim \$D}\qty[F(\bm{\theta},\bm{\zeta})],
\end{equation*}
where $F(\cdot,\cdot)$ is differentiable and nonconvex, the data $\bm{\zeta}$ is drawn from an unknown distribution $\$D$, and $\bm{\theta}$ is model weight. For analysis, we first make several mild assumptions.
\begin{assumption}[$L$-smoothness]\label{asm:Lsmooth} 
	The function $f(\cdot)$ is $L$-smooth w.r.t. the parameter, i.e., $\exists L > 0$, we have:
	\[
	\norm{\nabla f(\xmi{1}) - \nabla f(\xmi{2})}_2 \leq L \norm{\xmi{1}- \xmi{2}}_2, \quad \forall \xmi{1}, \xmi{2}.
	\]
\end{assumption}

\begin{assumption}[Boundedness]\label{asm:boundVar} 
	The  gradient estimation $\*g_k$ on each GPU node  is unbiased, i.e., $\EE[\*g_k] = \nabla f(\xmi{k})$, 
	and its magnitude and variance are bounded: 
	\[
	\EE \norm{\*g_k}_{\infty} \leq \ginf,  \quad	\EE\qty[\norm{\nabla f(\xmi{k}) - \*g_k }_2^2] \leq \sigma^2. 
	\]
\end{assumption}

\begin{assumption}[Bit-length]\label{asm:p-bit} 
Support that the compression operations in \eqref{eq:comadadaspressor} and  \eqref{adadsa} respectively use   $p$-bit with a scalar $s$ and $p_e$-bit with a scalar $s_e$.  With the proper $p$, $p_e$, $s$, and $s_e$,  there exist a constant $0\leq \alpha <1$ such that $(1-\alpha)s c_\infty + s/2s_e \leq  2^p $ and $T_c\alpha  \beta  s_e c_\infty \leq  2^{p_e}$, where $\beta$ is given  in Eqn.~\eqref{dada}.
\end{assumption}

Assumptions~\ref{asm:Lsmooth} and \ref{asm:boundVar} are mild and frequently used in the analysis of general nonconvex problems~\cite{guo2021novel,arjevani2019lower,zhou2022win,zhou2024towards, xie2024optimization}. The bounded gradient assumption, while commonly utilized, is primarily necessary for the convergence analysis of Adam-family methods and is not specifically required by EFC. When considering only the SGD case, as with some EFC-based methods like EF21, \method{}'s convergence guarantee also holds without relying on the bounded gradient assumption.
It is important to note that removing the bounded gradient assumption for the convergence analysis of Adam-family methods remains an open and challenging problem~\cite{zhang2022adam,li2024frac}. We acknowledge this limitation and leave it for future work.

Assumption~\ref{asm:p-bit} quantifies the expected precision loss introduced by the two compression operations within Algorithm~\ref{alg:BitCom}.  1) For the $p$-bit gradient compression operation, we assume that in the worst-case scenario, the precision degradation does not exceed $\alpha$ times the upper bound of the gradient norm, denoted as $\alpha c_\infty$, where $0\leq \alpha<1$. This is similar to the necessary condition $\norm{\tilde{\*e}^n_{k+1}}^2 \leq \alpha \norm{\*g^n_k}^2$  used in the analysis of biased SGD~\cite{ajalloeian2020convergence}. This assumption ensures that in extreme cases, the compression error does not surpass the gradient itself, thereby preserving a portion of the gradient information and preventing the complete loss of gradient direction; 
2) For the  $p$-bit error compression operation, which lacks a feedback mechanism, we assume that the precision degradation introduced by the error compression does not accumulate excessively during the error reset period. By appropriately setting hyper-parameters, such as in our practical choice of $p_e = 8$ and $p=4$, this assumption ensures that the precision degradation from $p_e$-bit error compression remains less than that from $p$-bit gradient compression. This is a relatively weak condition and can be easily met.

Then, we investigate the convergence performance of \method{}-integrated SGD and Adam-type optimizers in turn. 

\vspace{-2mm}
\subsection{\method{}-integrated SGD}
For SGD, its algorithmic steps are as:
\begin{equation}\label{eq:sgd-adam}
	\begin{split}
    \text{SGD:}&\quad \xmi{k+1} =    \xmi{k}  - \eta \tilde{\*g}_k,
    \end{split}
\end{equation}
where the gradient $\tilde{\*g}_k$ is given  in Eqn.~\eqref{adadaasda}.
Now we are ready to analyze \method{}-integrated SGD and summarize the results in Theorem~\ref{thm:sgd} with its proof in Appendix~\ref{sec:proof:SGD}.

\begin{theorem}[SGD Convergence]\label{thm:sgd}
	Suppose that Assumptions~\ref{asm:Lsmooth}, \ref{asm:boundVar}, and~\ref{asm:p-bit} hold. 
Let $s_e = \Omega(\epsilon^{-4})$ and $\eta = \order{\epsilon^2}$ in \method{}-integrated SGD. Then, after $T= \Omega\qty(\epsilon^{-4})$ iterations, we have: 
\[
\frac{1}{T} \sum_{k=0}^T \EE \norm{\nabla f\qty({\bm{\theta}}_{k})}^2  \leq \order{{\epsilon^2}}.
\] 
That is, the stochastic gradient complexity to find an $\epsilon$-accurate first-order stationary point is $\order{\epsilon^{-4}}$.
\end{theorem}
Theorem~\ref{thm:sgd} demonstrates that \method{}-integrated SGD and the original SGD exhibit identical convergence rates, with stochastic gradient complexity (i.e., the number of gradient evaluations) relative to $\epsilon$, achieving the theoretical lower bound~\cite{arjevani2019lower}. The advantageous properties of \method{}-integrated SGD arise from two key factors: 1) The refined EFC in \method{} ensures that gradient compression errors do not accumulate over iterations. As shown in Eqn.~\eqref{eq:acm_err}, both single-step and multi-step errors remain of the same order of magnitude; 2) \method{}'s periodic error reset mechanism prevents the compression errors, when compressing the compensation error $\tilde{\*e}^n_{k+1}$ to its 8-bit form ${\*e}^n_{k+1}$, from accumulating over time. This reset mechanism ensures that these compression errors are eliminated after a finite number of steps.
These properties ensure that the distance between the training sequences of \method{}-integrated SGD and the original SGD remains tightly bounded.

\begin{table*}[t!]
	\centering
		\caption{Comparison of \method{} with previous SoTA methods across various metrics, including gradient complexity, communication time, memory overhead, collective communication support, and sharding support. Here, $\Psi$ represents the number of model parameters, $N_d$ is the number of distributed nodes, $B$ is the communication bandwidth (bytes/s), and $r$ is the low-rank parameter specific to PowerSGD. We consider the mixed-precision setting for memory overhead.}
		\label{tab:comparison}
		\setlength{\tabcolsep}{10pt} 
		\renewcommand{\arraystretch}{1.6}
		{ \fontsize{8.3}{3}\selectfont{
				\begin{tabular}{l|c|l|l|c|c}
					\toprule
				Methods &  \begin{tabular}[c]{@{}c@{}}Gradient \\ Complexity\end{tabular} & Communication  Time & {\begin{tabular}[c]{@{}c@{}}Memory  \\ Consumed\end{tabular}}  & {\begin{tabular}[c]{@{}c@{}} Collective\\ Comm. \end{tabular}} &  {\begin{tabular}[c]{@{}c@{}} Opt. + Grad.\\ Sharding \end{tabular}}  \\
         \midrule
          EF~\cite{seide20141} &  $\order{\epsilon^{-4}}$ 	& 	$2.5\Psi N_d/B$ & $10{\Psi}$    &  \ding{55} & \ding{55} \\
    \midrule
    EF21~\cite{richtarik2021ef21} & $\order{\epsilon^{-4}}$ 	&	$2.5\Psi N_d/B$   & $10{\Psi}$ & \ding{55} & \ding{55}   \\
    \midrule
    1-bit Adam~\cite{tang20211} & $\order{\epsilon^{-4}}$ & $0.325{\Psi} \qty(N_d-1)/\qty(BN_d)$ &  $18{\Psi} + 2{\Psi}/N_d$ & \ding{51}	& \ding{55}   \\
    \midrule
    1-bit LAMB~\cite{li20221} & $\order{\epsilon^{-4}}$ & 	$0.325{\Psi} \qty(N_d-1)/\qty(BN_d)$  & $22{\Psi}+ 2{\Psi}/N_d$   &  \ding{51}	&\ding{55} \\ 
    \midrule
    PowerSGD~\cite{vogels2019powersgd} & \ding{55} & 	$4r\sqrt{\Psi} \qty(N_d-1)/\qty(BN_d)$ & $14{\Psi}+ 2r\sqrt{\Psi}$  & \ding{51}	& \ding{51}   \\ 
    \midrule
     Modified EF-SGD &  $\order{\epsilon^{-4}}$ 	& 	$2.25{\Psi} \qty(N_d-1)/\qty(BN_d)$ & $4{\Psi} + 6\Psi/N_d$   &  \ding{51} &\ding{51}   \\
    \midrule
     Modified  EF21-SGD & $\order{\epsilon^{-4}}$ & 	$2.25{\Psi} \qty(N_d-1)/\qty(BN_d)$ & $4{\Psi} + 10\Psi/N_d$  & \ding{51}	& \ding{51}   \\
     \midrule
		Adam~\cite{kingma2014adam}& $\order{\epsilon^{-4}}$ & 	$4{\Psi} \qty(N_d-1)/\qty(BN_d)$ & $2{\Psi} + 14\Psi/N_d$    & \ding{51}	& \ding{51}  \\
    \midrule
    SGD & $\order{\epsilon^{-4}}$ & 	$4{\Psi} \qty(N_d-1)/\qty(BN_d)$ &  $2{\Psi} + 6\Psi/N_d$ & \ding{51}	& \ding{51}   \\
     \midrule
     Adam-Zero++~\cite{wang2023zero++} & \ding{55}& 	$1.5{\Psi} \qty(N_d-1)\qty(BN_d)$ & $2{\Psi} + 14\Psi/N_d$   & \ding{51}	& \ding{51}   \\
     \midrule
     \method{}-SGD (ours) & $\order{\epsilon^{-4}}$& 	$2.25{\Psi} \qty(N_d-1)/\qty(BN_d)$& $3{\Psi} + 6\Psi/N_d$      & \ding{51}	& \ding{51} \\
     \midrule
     \method{}-Adam (ours)  & $\order{\epsilon^{-4}}$& 	$2.25{\Psi} \qty(N_d-1)/\qty(BN_d)$  & $3{\Psi} + 14\Psi/N_d$  &  \ding{51}	&\ding{51}   \\ \midrule
     \method{}-Zero++ (ours)  & $\order{\epsilon^{-4}}$& 	$1.5{\Psi} \qty(N_d-1)/\qty(BN_d)$  & $3{\Psi} + 14\Psi/N_d$  &  \ding{51}	&\ding{51}   \\
   \bottomrule 
\end{tabular}	
\vspace{-5mm}
}}
\end{table*}

\subsection{\method{}-integrated Adam-family Optimizers}
The algorithmic steps for Adam-type optimizers are
\begin{equation}\label{eq:adam}
    \text{Adam Family:}
    \begin{cases}
    	\mmi{k} =  (1-\betai{1})  \mmi{k-1} +  \betai{1} \tilde{\*g}_k, \\
    	\bm{\eta}_k =   \eta \times v\qty(\tilde{\*g}_0,\cdots,\tilde{\*g}_k), \\
    	\xmi{k+1} =    \xmi{k}  - \bm{\eta}_k \circ \mmi{k},
    \end{cases}
\end{equation}
where $\circ$ is the element-wise product, $\beta_1\in(0,1)$, and $\mmi{0}=\bm{0}$.
For Adam Family, its operation $v\qty(\cdot)$ is to compute the pre-conditioner, e.g., the inverse of Adam's second-order moment,  $v\qty(\tilde{\*g}_0,\cdots,\tilde{\*g}_k)=
1/\sqrt{\vmi{k}+\delta}$, where $\vmi{k} = (1-\beta_2) \vmi{k-1} + \beta_2 \tilde{\*g}_k^2$ with $\beta_2\in(0,1)$ and $\vmi{0}=\bm{0}$.   So by choosing different operation $v\qty(\cdot)$, Adam Family contains many prevent optimizers~\cite{guo2021novel}, e.g., Adam,  Aadafactor~\cite{shazeer2018adafactor}, and AdamW~\cite{loshchilov2018decoupled}, etc.
Before showing the main results, we introduce an assumption regarding the operation $v\qty(\cdot)$. It should be noted that this is not strictly an assumption, as the condition can be readily satisfied by appropriately setting a hyper-parameter for $v\qty(\cdot)$.
\begin{assumption}[Pre-conditioner]\label{asm:pre_cond} 
The pre-conditioner is element-wise bounded, i.e.,  $c_l \leq \|v(\cdot)\|_{\infty} \leq c_u$. Additionally, the element-wise difference of successive pre-conditioners is bounded:
\[
\norm{v\qty(\tilde{\*g}_0,\cdots,\tilde{\*g}_k)  - v\qty(\tilde{\*g}_0,\cdots,\tilde{\*g}_{k-1})}_\infty \leq {\beta_1c_u}, 
\]
where $\beta_1$, as defined in Eqn.~\eqref{eq:adam}, represents the momentum for Adam-family methods. 
\end{assumption}
Assumption~\ref{asm:pre_cond} ensures that the pre-conditioner maintains appropriate bounds and does not fluctuate significantly during training. This requirement is commonly satisfied by several Adam-type optimizers. For instance, in Adam and AdamW, setting $v\qty(\tilde{\*g}_0,\cdots,\tilde{\*g}_k)=1/\sqrt{\vmi{k}+\delta}$, with $\vmi{k} = (1-\beta_2) \vmi{k-1} + \beta_2 \tilde{\*g}_k^2$, yields $c_u = 1/\sqrt{\delta}$. By choosing $\beta_2 = \mathcal{O}(\beta_1)$, the difference bound is always satisfied. In practice, it is typical to set $\beta_1=0.1$ and $\beta_2 \in [0.001,0.05]$, thereby complying with Assumption~\ref{asm:pre_cond}. For further details on other Adam-type optimizers and their adherence to Assumption~\ref{asm:pre_cond}, please refer to the discussion in~\cite{guo2021novel}.

\begin{theorem}\label{thm:adam}
	Suppose   Assumptions~\ref{asm:Lsmooth}, \ref{asm:boundVar}, \ref{asm:p-bit}, and~\ref{asm:pre_cond} hold. 
	Let $s_e = \Omega(\epsilon^{-4})$, $\eta =  \order{\epsilon^2}$, and $\beta_1 \!=\! \order{\epsilon^2}$ in \method{}-integrated Adam-type optimizers, then after $T= \Omega\qty(\epsilon^{-4})$ iterations, the following inequality  holds:
 \begin{equation}\label{adasdafsa}
 	\begin{split}
 		\frac{1}{T} \sum_{k=0}^{T-1}  \EE \left[ \left\|   \nabla f(\xmi{k})   \right\|^2+ \frac{  1 }{4}    \left\|    \mmi{k} \right\|^2 	\right] 
 		\leq \epsilon^2. 
 	\end{split}
 \end{equation} 
 That is, the stochastic gradient complexity is  $\order{\epsilon^{-4}}$.  
\end{theorem}
{Theorem~\ref{thm:adam} reveals an important insight: \emph{integrating \method{} into various adaptive optimizers does not compromise their convergence speed}. 
Furthermore, Theorems~\ref{thm:sgd} and~\ref{thm:adam}  collectively show that \method{}-integrated optimizers exhibit the same stochastic gradient complexity as their vanilla counterparts. For instance, when aiming for an $\epsilon$-accurate first-order stationary point ($\epsilon$-FOSP) in a non-convex optimization problem, the complexity for both SGD and Adam-type optimizers, as well as their \method{}-integrated versions, remains $\order{\epsilon^{-4}}$. This illustrates that \method{} preserves the convergence performance while utilizing low-bit gradients, thereby significantly enhancing communication efficiency.}


\subsection{Comparison of Communication-Efficient Methods}
Here we compare \method{} with previous efficient training methods across five aspects: gradient complexity, communication time, memory overhead,  collective communication support, and sharding support. The specific results are summarized in Table~\ref{tab:comparison}. Consistent with the Zero sharding method~\cite{rajbhandari2020zero}, we consider partitioning the optimizer states and gradients, i.e., the scenario of Zero2.
Here, $\Psi$ denotes the number of model parameters, $N_d$ represents the number of nodes in distributed training (for parameter-server architecture, $N_d$ refers to the number of nodes computing local gradients), $B$ is the communication bandwidth in bytes per second, and $r$ is the low-rank parameter specific to PowerSGD~\cite{vogels2019powersgd}.
We also evaluate modified EF and modified EF21, which adapt the original EF~\cite{seide20141} and EF21~\cite{richtarik2021ef21} methods for the popular sharding framework despite being originally designed for parameter-server architecture. 

For memory computation, we consider a mixed-precision training setting where 16-bit parameters and 16-bit gradients are present in memory (each consuming $2\Psi$ bytes). Additionally, for SGD and Adam optimizers, there is a 32-bit parameter copy (consuming $4\Psi$ bytes). Furthermore, Adam optimizers require an additional $8\Psi$ bytes to store first-order and second-order moments. The 1-bit LAMB~\cite{li20221} method requires an additional $4\Psi$ bytes for another second-order moment. EFC-based methods need extra memory to store the error (16-bit, consuming $2\Psi$ bytes, decoupled from optimizer states), whereas \method{} only requires an additional $\Psi$ bytes, as it stores the error in 8-bit format. Modified EF21 additionally needs to store a shared global error variable, which consumes $4\Psi/N_d$ bytes. PowerSGD, in addition to using EFC, requires extra memory to store a 16-bit low-rank matrix, consuming $2r\sqrt{\Psi}$ bytes. For sharding scenarios, we consider splitting gradients and optimizer states but not the 16-bit model parameters to avoid additional communication overhead.

Communication time in FSDP setting usually involves two parts: gradient communication (reduce-scatter) and parameter synchronization (all-gather). For the parameter-server architecture, gradients are sent to the server node, processed, and then parameters are returned to the nodes. The total communication time is $\qty(b_g + b_w) \Psi N_d / 8B$, where $b_g$ and $b_w$ are the bits used for communicating gradients and parameters/weights, respectively. For example, for 1-bit Adam and 1-bit LAMB, $b_g = 1$ and $b_w = 1$, noting that the first 10\% of iteration steps use full-precision communication as a warm-up. For EF~\cite{seide20141} and EF21~\cite{richtarik2021ef21}, $b_g = 4$ and $b_w = 16$.

For methods supporting collective communication, the total communication time is given by $\qty(b_g + b_w) \Psi (N_d - 1) / \qty(8N_dB)$. In this scheme, each node's communication volume per step is $b \Psi / N_d$, requiring $(N_d - 1)$ steps to complete a full gradient or parameter exchange. Here, $b$ represents the bit size used for communication. Specifically, for Adam and SGD, $b_g = 16$ and $b_w = 16$. For Zero++ (combined with Adam or LoCo), $b_g = 4$ and $b_w = 8$. For \method{}, modified EF, and modified EF21, $b_g = 4$ and $b_w = 16$. For PowerSGD, the communication volume is $br\sqrt{\Psi} / N_d$, while other aspects remain similar to SGD.

\begin{table*}[t!]
	\begin{center}
		\caption{Comparison of 16-bit Adam and 4-bit \method{}-integrated Adam. The table presents evaluation results on standard benchmarks for LLM capabilities, including accuracy and success rates across tasks such as commonsense reasoning, world knowledge, mathematics, and coding. Higher values indicate better performance.
        \vspace{-2mm}
  }
		\label{tab:eval}
		\setlength{\tabcolsep}{2.8pt} 
		\renewcommand{\arraystretch}{3.5}
		{ \fontsize{8.5}{3}\selectfont{
				\begin{tabular}{c|c|cccccccccccc|c}
					\toprule
				Model & Optimizer 	& MMLU & HellaS & WinoG & PIQA & Arc-e & Arc-c & NQ & TriQA & HumanE & MBPP & Math & GSM8K &Avg.  \\
         \midrule
         \multirow{2}{*}{\begin{tabular}[c]{@{}c@{}}LLAMA2 \\ (7B)\end{tabular}}	& 	Adam   &  46.0\% & \textbf{74.3}\% & \textbf{62.3}\% & 77.7\% & \textbf{59.2}\% & 43.3\% & \textbf{14.0}\% & \textbf{51.1}\% & 13.4\% & \textbf{15.8}\% & 3.1\% & 16.8\% & \textbf{39.75}\% \\
    &	Adam+\method{}   &  \textbf{46.2}\% & 74.3\% & 62.2\% & \textbf{77.8}\% & 58.6\% & \textbf{44.1}\% & 13.4\% & 50.8\% & \textbf{13.5}\% & 14.8\% & \textbf{3.2}\% & \textbf{17.5}\% & 39.70\%   \\
  \midrule
		\multirow{2}{*}{\begin{tabular}[c]{@{}c@{}}LLAMA2 \\ (13B)\end{tabular}}	&	Adam   &  55.2\% & 77.7\% & 64.0\% & 79.8\% & 74.7\% & \textbf{58.9}\% & \textbf{15.6}\% & 56.9\% &\textbf{20.1}\% & 28.6\% & 5.2\% & 29.1\% & 47.15\% \\
   & 	Adam+\method{}   &  55.2\% & 77.7\% & \textbf{64.3}\% & 79.8\% & \textbf{74.8}\% & 58.6\% & 15.5\% & \textbf{57.2}\% & 19.5\% & 28.6\% & \textbf{5.4}\% & \textbf{30.1}\% & \textbf{47.22}\%  \\
  \midrule
   \multirow{2}{*}{\begin{tabular}[c]{@{}c@{}}Mixtral \\ (8$\times$7B)\end{tabular}}	 & AdamW   &  70.1\% & \textbf{82.0}\% & \textbf{70.1}\% & \textbf{83.3}\% & 94.0\% & 86.1\% & 31.1\% & \textbf{65.5}\% & 34.7\% & \textbf{42.4}\% & 22.9\% & \textbf{70.4}\% & 62.71\% \\
 &	AdamW+\method{}   &  \textbf{70.3}\% & 81.7\% & 69.7\% & 83.1\% & \textbf{94.7}\% & \textbf{86.8}\% & \textbf{31.1}\% & 64.6\% & \textbf{36.6}\% & 41.6\% &\textbf{23.2}\% & 69.5\% & \textbf{62.74}\%\\
    \bottomrule 
\end{tabular}		
		}}
	\end{center}
 \vspace{-6mm}
\end{table*} 

From the results in Table~\ref{tab:comparison}, it is evident that \method{} demonstrates superior properties compared to other methods. \method{} supports both ring-based and sharding strategies, which keeps communication costs stable regardless of the number of GPUs and reduces memory overhead. This is particularly advantageous over methods like EF21 and PowerSGD, which incur higher memory usage due to a lack of sharding support. Therefore, \method{} is more suitable for training large-scale models, offering both efficient communication and small memory costs. 
Although \method{} does not show a clear advantage over Zero++ in terms of communication time and memory overhead, the latter lacks convergence guarantees, potentially impacting training quality, as verified in Fig.~\ref{fig:curve} and Table~\ref{tab:lowbit_compare}. Methods without convergence guarantees tend to underperform.

Despite the ability of original EF and EF21 to adapt to the sharding framework, they consume additional memory compared to \method{}. For instance, modified EF21 requires extra storage for a shared global error variable, increasing the optimizer's 32-bit state size. Moreover, our experiments indicate that merely adapting EFC to the sharding framework can lead to significant performance degradation. Without the improvements proposed in \method{}, such as error averaging and error resetting, training performance can suffer drastically, even failing in pre-training experiments (see Sec.~\ref{sec:sota}) or showing subpar results in fine-tuning (see Table~\ref{tab:abla_study}).

Finally, compared to the original Adam or SGD, \method{} achieves significant communication efficiency improvements with small additional memory overhead while maintaining the same convergence speed. In large-scale experiments, the speedup can reach 1.4x or more, with less than 10\% additional memory overhead. This memory overhead is even negligible in some cases, such as in LLM with large batch sizes and token lengths (see Table~\ref{tab:speedup}~and~\ref{tab:mem}).

\begin{table}[ht!]
	\centering
		\caption{The fine-tuning losses of 4-bit \method{} and the 16-bit communication based optimizers on LLAMA2 and Mixtral.}
		\label{tab:valloss}
		\setlength{\tabcolsep}{7pt} 
		\renewcommand{\arraystretch}{3}
		{ \fontsize{8.5}{3}\selectfont{
				\begin{tabular}{c|c|c|cc}
					\toprule
				Model & Optimizer	& Loss & Baseline   & \method{}  \\
         \midrule
		\multirow{2}{*}{\begin{tabular}[c]{@{}c@{}}LLAMA2 \\ (7B)\end{tabular}} & \multirow{2}{*}{Adam} 	&	Train   &  1.688 & {1.688}  \\
   & & Val.   &  1.503 & {1.503}   \\ 
    \midrule
    \multirow{2}{*}{\begin{tabular}[c]{@{}c@{}}LLAMA2 \\ (13B)\end{tabular}} & \multirow{2}{*}{Adam}	&	Train  &  1.612 & {1.612}  \\
   & & Val.   &  1.415 & {1.415}  \\ 
   \bottomrule
 \multirow{2}{*}{\begin{tabular}[c]{@{}c@{}}Mixtral \\ (8$\times$7B)\end{tabular}} & \multirow{2}{*}{AdamW}	&	Train   &  0.7041 & \textbf{0.7030}  \\
   & & Val.   &  0.7038 & \textbf{0.7030} \\ \midrule
   \multirow{2}{*}{\begin{tabular}[c]{@{}c@{}}Mixtral \\ (8$\times$7B)\end{tabular}}	& \multirow{2}{*}{Adafactor} &  	Train   &  0.6123 & \textbf{0.6117}  \\
   & & Val.   &  \textbf{0.6206} & 0.6207  \\
   \bottomrule 
\end{tabular}	
\vspace{-4mm}
}}
\end{table}

\begin{table*}[t]
	\begin{center}
		\caption{Performance comparison of low-bit communication methods. Metrics (accuracy and success rates) are reported for LLaMA2-7B fine-tuned on the Alpaca-GPT4 dataset across commonsense reasoning, world knowledge, mathematics, and coding, where higher values indicate better performance. \vspace{-2mm}
 }\label{tab:lowbit_compare}
		\setlength{\tabcolsep}{6.0pt} 
		\renewcommand{\arraystretch}{3}
		{ \fontsize{8.5}{3}\selectfont{
				\begin{tabular}{l|cccccccccc|c}
					\toprule
				Method& {HumanE} & MBPP & GSM8K & NQ &{HellaS} & {Arc-e} & {Arc-c} & {PIQA} & WinoG & {MMLU}  &{Avg.}  \\
         \midrule
         Adam (16-bit) & 16.4\% &20.0\% &15.7\% &14.6\% &75.4\% &73.5\% &52.2\% &78.6\% &61.4\% &48.0\% &45.6\% \\
         \midrule
         0/1 Adam (4-bit) & 15.2\% &21.0\% &11.8\% &13.3\% &74.9\% &74.6\% &50.5\% &78.0\% &61.2\% &47.7\% &44.8\%  \\
        4-bit Adam & 15.3\% &19.4\% &13.4\% &14.4\% &74.5\% &74.8\% &47.1\% &78.2\% &61.7\% &47.2\% &44.6\% \\
         4-bit LAMB  & 17.6\% &22.0\% &15.2\% &13.0\% &75.2\% &71.0\% &50.5\% &78.5\% &61.5\% &47.8\% &45.2\% \\
         Zero++ (4-bit) & 15.9\% &17.6\% &13.6\% &14.0\% &75.8\% &72.5\% &46.1\% &78.0\% &61.9\% &47.1\% &44.3\%\\
         Adam+\method{} (4-bit) & 16.5\% &21.8\% &15.6\% &13.9\% &76.1\% &72.9\% &51.4\% &78.3\% &61.8\% &48.0\% &45.6\%\\
    \bottomrule 
\end{tabular}		
		}}
	\end{center}
 \vspace{-6mm}
\end{table*} 

\begin{table}[t!]
	\centering
		\caption{Comparison of pre-training loss for Sky-MoE with different data volumes and model sizes using \method{}-integrated Adam (4-bit) and 16-bit Adam.}
		\label{tab:moeloss}
		\setlength{\tabcolsep}{4pt} 
		\renewcommand{\arraystretch}{4}
		{ \fontsize{9.0}{3}\selectfont{
				\begin{tabular}{c|c|c|cc}
					\toprule
				Models & Token Size& Model Size & Adam   & \method{}  \\
         \midrule
		\multirow{3}{*}{Sky-MoE~\cite{zhao2024longskywork}} & 10B 	&	$8\times0.1$B   &  \textbf{2.635} & {2.636}  \\
  \cline{2-5}
   & 30B 	&	$8\times0.1$B   &  2.477 & \textbf{2.475}  \\
  \cline{2-5}
  & 300B 	&	$8\times0.3$B   &  2.105 & \textbf{2.102}  \\
   \bottomrule 
\end{tabular}		
}}
\vspace{-5mm}
\end{table}

\begin{figure*}[ht]
    \centering
    \begin{minipage}{\textwidth}
        \subfigure[]{
            \includegraphics[width=0.32\linewidth]{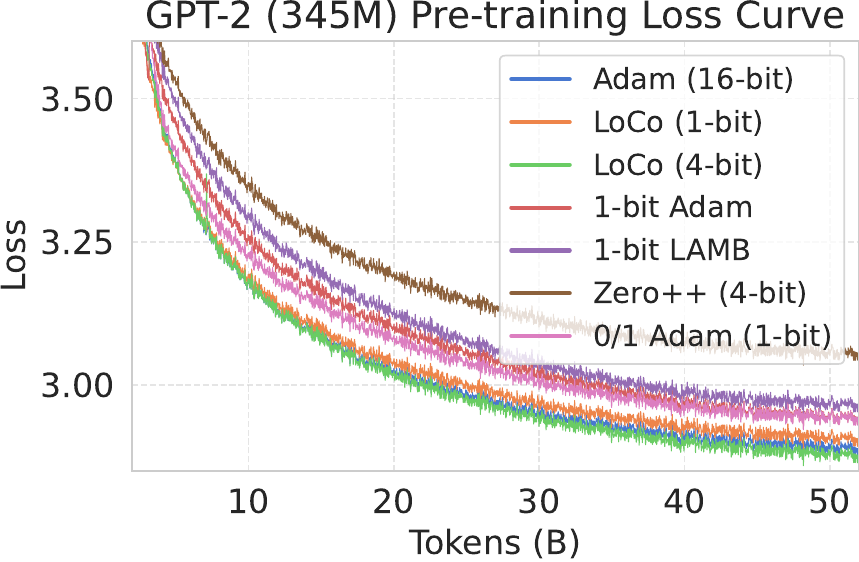}
        }
        \subfigure[]{
            \includegraphics[width=0.32\linewidth]{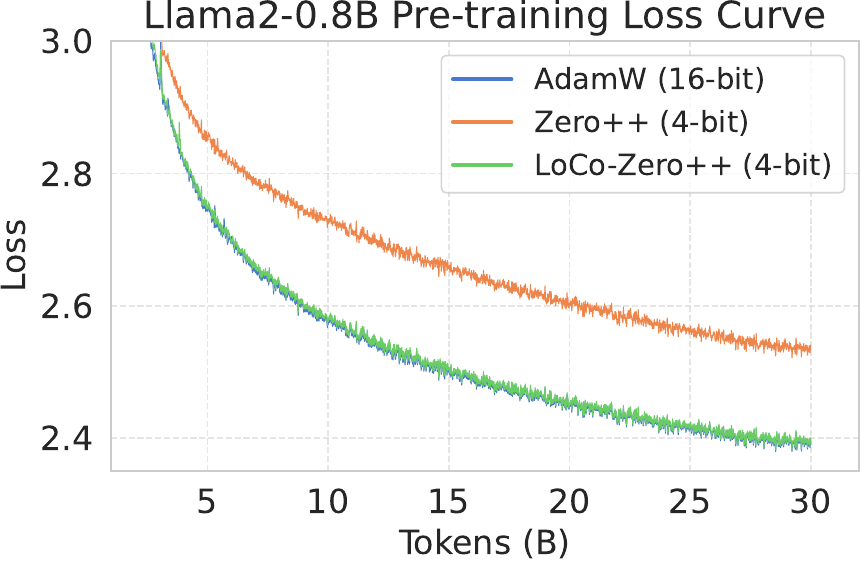}
        }
        \subfigure[]{
            \includegraphics[width=0.32\linewidth]{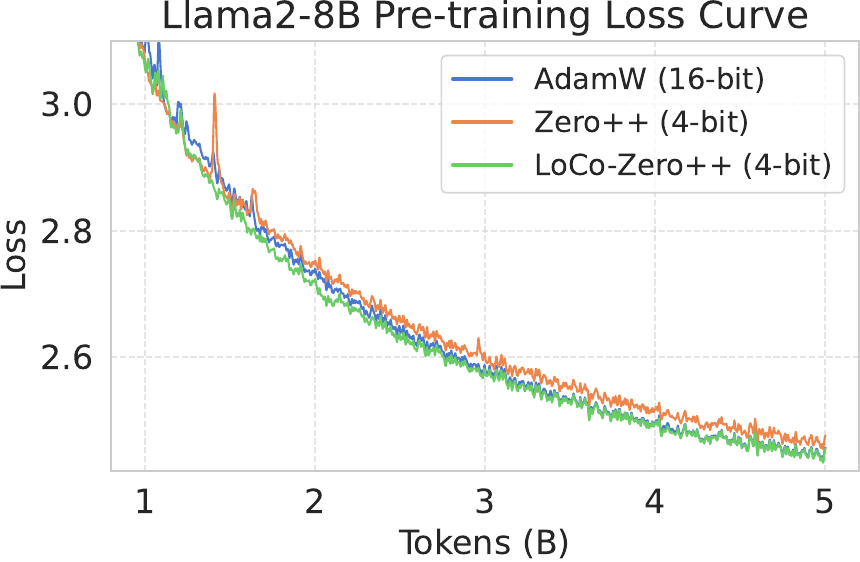}
        }
    \end{minipage}
    \vspace{-3mm} 
    \caption{ Loss curves of various low-bit optimization methods. (a) Training results on GPT2-345M with 52B tokens from the OpenWebtext dataset, showing that 4-bit \method{} achieves performance comparable to 16-bit Adam. (b) Training results on LLaMA2-0.8B with 30B tokens from RedPajama-v2, where \method{}-Zero++ achieves training quality on par with 16-bit AdamW and outperforms Zero++. (c) Training results on LLaMA2-8B with 5B tokens, illustrating the effectiveness of \method{}-Zero++ in maintaining high training quality even for larger model sizes.}\label{fig:curve} 
    \vspace{-5mm}
\end{figure*}

\vspace{-2mm}
\section{Experiments}
To test  \method{}, we first compare it with several representative baselines, including low-bit optimizers with error-feedback like 1-bit Adam~\cite{tang20211}, and quantization method like Zero++ \cite{wang2023zero++}. Moreover, we further compare \method{} with the widely used  16-bit optimizers, e.g., Adam\cite{kingma2014adam}, AdamW~\cite{loshchilov2018decoupled}, and Adafactor~\cite{shazeer2018adafactor}. Subsequently, we examine the training speed of \method{} across diverse model architectures, model sizes, GPU server configurations, and large-model training frameworks. Finally, we investigate the effect of each key component in \method{}.  See detailed experimental settings in Appendix~\ref{eval_details}.

\vspace{-3mm}
\subsection{Results on  \method{}-Integrated Optimizers }  
We integrate \method{} (4-bit) into various optimizers, including Adam, AdamW \cite{loshchilov2018decoupled}, and Adafactor, and compare with the corresponding 16-bit counterparts. For evaluation, we use these optimizers to train two advanced LLMs, LLAMA2 and Mixtral,  under collective communication. 
For LLAMA2, we follow \cite{zhang2024tinyllama} to use the RedPajama subset for fine-tuning on 8B tokens. For  Mixtral, following~\cite{zhang2023llamaadapter}, we adopt ultra-chat-200k dataset~\cite{ding2023enhancing} to fine-tune for one epoch.  These fine-tuning protocols are to align with the established benchmarks, and provide convincing assessment to  \method{}.

\textbf{Fine-tuning results on LLAMA2 and Mixtral (8x7B).} 
From Table~\ref{tab:eval}, one can see that when evaluated on downstream tasks, 4-bit \method{}-integrated optimizers share comparable performance as their corresponding 16-bit ones, and even occasionally exceed them.
Table~\ref{tab:valloss}  shows that 4-bit \method{}-integrated optimizers enjoy similar training and validation losses with their corresponding 16-bit counterparts on LLAMA2 and Mixtral, indicating their similar convergence speed and training quality. 
All these results affirm the effectiveness of  \method{} on the large-scale model training.

\vspace{-3mm}
\subsection{SoTA Comparison Under Low-bit Communication}\label{sec:sota}
Here we compare \method{} with  communication-efficient methods including 1-bit Adam, 1-bit LAMB~\cite{li20221}, 0/1 Adam~\cite{lu2022maximizing} and Zero++. Among them, 1-bit Adam, 1-bit LAMB, and 0/1 Adam cannot be directly applied to FSDP settings.  Zero++ compresses its gradient into low-bit ones without any error-feedback strategy and could suffer from information loss. 

\begin{table*}[h]
	\begin{center}
 \caption{Comparison of fine-tuning results using LoRA on LLaMA2-7B. Metrics (accuracy and success rates) are reported on the Alpaca-GPT4 dataset, with higher values indicating better performance.\vspace{-2mm}}\label{tab:powersgd}
  \label{tab:lora}
		\setlength{\tabcolsep}{7.0pt} 
		\renewcommand{\arraystretch}{3}
		{ \fontsize{8.5}{3}\selectfont{
				\begin{tabular}{l|ccccccccc|c}
					\toprule
				Method& {HumanE} & MBPP & GSM8K &{HellaS} & {Arc-e} & {Arc-c} & {PIQA} & WinoG & {MMLU}  &{Avg.}  \\
         \midrule
         AdamW (16-bit) & 13.4\% &17.4\% &16.2\% &73.9\% &59.8\% &44.4\% &77.9\% &62.6\% &45.9\% &45.7\% \\
         \midrule
         PowerSGD & 13.2\% &19.0\% &16.9\% &73.4\% &58.1\% &43.1\% &77.2\% &62.5\% &45.7\% &45.4\%  \\
        AdamW+\method{}   & 13.0\% &17.2\% &16.7\% &73.8\% &59.9\% &44.4\% &77.6\% &62.4\% &45.9\% &45.7\% \\
    \bottomrule 
\end{tabular}		
		}}
	\end{center}
 \vspace{-2mm}
\end{table*}

\begin{table*}[t!]
	\begin{center}
		\caption{Training speed (throughput, tokens/s) investigation of \method{} (4-bit) and Adam (16-bit) in prevalent Megatron-LM framework for efficient multi-node LLMs training with different model sizes, GPU number, and the node connection.}
		\label{tab:speedup}
		\setlength{\tabcolsep}{6pt} 
		\renewcommand{\arraystretch}{3.3}
		{ \fontsize{8.5}{3.2}\selectfont{
				\begin{tabular}{l|ccc|ccc|ccc}
					\toprule
				& \multicolumn{3}{c|}{32$\times$ NVIDIA A100 (RoCE v2)} & \multicolumn{3}{c|}{64$\times$ NVIDIA A100 (RoCE v2)}  & \multicolumn{3}{c}{128$\times$ NVIDIA A100 (RoCE v2)}     \\ \cline{2-10} 
					Model  & \begin{tabular}[c]{@{}c@{}}Adam\\ (tokens/s)\end{tabular}  & \begin{tabular}[c]{@{}c@{}}LoCo \\ (tokens/s)\end{tabular}  & \begin{tabular}[c]{@{}c@{}}Speedup \\ (\%)\end{tabular}    & \begin{tabular}[c]{@{}c@{}}Adam  \\ (tokens/s)\end{tabular}  & \begin{tabular}[c]{@{}c@{}}  LoCo \\ (tokens/s)\end{tabular}  & \begin{tabular}[c]{@{}c@{}}Speedup \\ (\%)\end{tabular}   & \begin{tabular}[c]{@{}c@{}}Adam \\ (tokens/s)\end{tabular}  & \begin{tabular}[c]{@{}c@{}}LoCo \\ (tokens/s)\end{tabular}  & \begin{tabular}[c]{@{}c@{}}Speedup \\ (\%)\end{tabular}   \\ \midrule
				LLAMA2 (7B)  & 57230.2 & 65376.3&  \textbf{14.23} &  108680.5 & 127263.1 & \textbf{17.10} &  212373.9 & 251701.9 & \textbf{18.50}   \\
    
Mistral (7B) &   55947.3 & 64123.7 &\textbf{14.61} &  105198.2 & 125422.7 & \textbf{19.22} &   206053.7 & 247468.3 & \textbf{20.10}   \\ 
LLAMA2 (13B) &  30555.9 & 35683.2 &\textbf{16.78} &   43941.6 & 55322.9 & \textbf{25.90} &  83160.2& 108577.2  & \textbf{30.56}   \\  
                \midrule
\textbf{LLAMA2 (70B)} & \multicolumn{3}{c|}{\multirow{1}{*}{N/A, since the Data Parallel is $\bm{1}$}} &  2869.2 & 3803.2 & \textbf{32.55} &  5263.6& 7107.6  & \textbf{35.03}   \\  
                \bottomrule
				& \multicolumn{3}{c|}{32$\times$ NVIDIA A800 (Infiniband)} & \multicolumn{3}{c|}{64$\times$ NVIDIA A800 (Infiniband)}  & \multicolumn{3}{c}{128$\times$ NVIDIA A800 (Infiniband)}     \\ 
			\midrule
		LLAMA2 (7B) & 54186.8 & 65862.1&  \textbf{21.55} &  89555.4& 120625.6 & \textbf{34.69} &  161447.6 & 224887.7 & \textbf{39.29}   \\
    
  Mistral (7B) & 51896.8 & 63568.5&  \textbf{22.49} &  85334.5& 115355.6 & \textbf{35.18} &  155308.7 & 217494.4 & \textbf{40.04}   \\
    
LLAMA2 (13B)& 30682.9 & 38226.1 &\textbf{24.58} &  49907.4 & 69409.0 & \textbf{39.08} &  90446.3 & 128649.6  & \textbf{42.24}   \\
                \bottomrule 
\end{tabular}		
		}}
	\end{center}
 \vspace{-4mm}
\end{table*}

\begin{table}[t!]
\centering
\caption{Comparison of peak memory of \method{} (4-bit) and Adam (16-bit) on various settings with 32 GPUs.}\label{tab:mem}
\setlength{\tabcolsep}{5.0pt} 
\renewcommand{\arraystretch}{3.2}
{ \fontsize{9.0}{3}\selectfont{
\begin{tabular}{l|c|cc}
\toprule
\multirow{2}{*}{Model} & \multirow{2}{*}{Framework} & \multicolumn{2}{c}{Peak Memory (GB)} \\ \cline{3-4}
&                                                                            & Adam              & Adam+\method{}             \\ \midrule  
Mixtral ($8\times 7$B)    & FSDP                                                                                           & 58.8              & 64.3 \\ \midrule   
LLAMA2-7B                & FSDP                                                                                           & 20.5             & 22.7 \\ \midrule 
Sky-MoE ($8\times 0.1$B)   & Megatron-LM                                                                                           & 72.3              & 72.7 \\ \midrule 
Sky-MoE ($8\times 0.3$B)   & Megatron-LM                                                                                           & 56.3              & 57.0 \\ \midrule 
LLAMA2-7B                & Megatron-LM                                                                                           & 44.0             & 48.1 \\ \midrule
LLAMA2-13B                & Megatron-LM                                                                                          & 68.3            & 74.5 \\ \bottomrule
\end{tabular}
}}
\vspace{-5mm}
\end{table}

\begin{table*}[t!]
	\begin{center}
\caption{Ablation study on \method{} components for LLaMA2-7B fine-tuned with Alpaca-GPT4. The table evaluates the impact of components including error-feedback, error compression, moving average on errors, and periodic error resetting. Metrics (accuracy and success rates) on downstream tasks are reported, with higher values indicating better performance.\vspace{-0.8em}}
		\label{tab:abla_study}
		\setlength{\tabcolsep}{2.8pt} 
		\renewcommand{\arraystretch}{3.5}
		{ \fontsize{8.5}{3}\selectfont{
				\begin{tabular}{l|cccc|cccccccccc|c}
					\toprule
				Method& {\begin{tabular}[c]{@{}c@{}}Error  \\ Feedback \end{tabular}} & {\begin{tabular}[c]{@{}c@{}}Error  \\ Cmpr. \end{tabular}} & {\begin{tabular}[c]{@{}c@{}}Err. Reset \\ Freq.\end{tabular}} & {\begin{tabular}[c]{@{}c@{}}Err. \\ Avg.\end{tabular}} & {HumanE} & MBPP & GSM8K & Math &{HellaS} & {Arc-e} & {Arc-c} & {PIQA} & WinoG & {MMLU}  &{Avg.}  \\
         \midrule
         \method{}1 & \ding{55} & N/A & N/A & \ding{55}& 49.0\% &50.0\% &67.2\% &23.0\% &81.7\% &94.0\% &87.4\% &82.6\% &68.7\% &70.8\% &67.4\%\\
         \method{}2 & \ding{51}   & \ding{51}  & N/A  & \ding{55}& 48.8\% &49.6\% &67.6\% &\textbf{23.7}\% &81.8\% &93.9\% &87.5\% &82.6\% &68.3\% &70.9\% &67.5\% \\
         \method{}3 & \ding{51} & \ding{51}  & N/A  & \ding{51} & 49.4\% &\textbf{50.4}\% &67.3\% &23.5\% &81.8\% &94.0\% &87.8\% &83.1\% &69.5\% &70.9\% &67.7\% \\
         \method{}4 & \ding{51}  & \ding{55} & 512 & \ding{51} & 51.8\% &50.0\% &66.5\% &22.9\% &\textbf{82.2}\% &93.8\% &87.8\% &\textbf{83.3}\% &\textbf{69.5}\% &70.4\% &67.8\% \\
         \method{}5 & \ding{51}  & \ding{51}  & 512 & \ding{51} & 51.8\% &50.0\% &\textbf{67.9}\% &23.2\% &81.8\% &94.0\% &\textbf{87.8}\% &83.1\% &68.8\% &71.0\% &67.9\% \\
         \method{}6& \ding{51}  & \ding{51}  & 128 & \ding{51} & \textbf{53.1}\% &50.2\% &67.3\% &23.5\% &81.8\% &\textbf{94.0}\% &87.4\% &83.2\% &69.1\% &\textbf{71.0}\% &\textbf{68.1}\%\\
    \bottomrule 
\end{tabular}		
		}}
	\end{center}
 \vspace{-1.5em}
\end{table*} 

\begin{table}[t!]
	\begin{center}
		\caption{Speedup achieved by 4-bit \method{} against 16-bit communication Adam within the PyTorch FSDP framework on an NVIDIA A800 cluster with Infiniband connectivity.\vspace{-2mm}}
		\label{tab:fsdp}
		\setlength{\tabcolsep}{8pt} 
		\renewcommand{\arraystretch}{3.5}
		{ \fontsize{8.5}{3.5}\selectfont{
				\begin{tabular}{c|c|ccc}
					\toprule
					Model& \begin{tabular}[c]{@{}c@{}}GPU \\ Number \end{tabular} & \begin{tabular}[c]{@{}c@{}}Adam \\ (tokens/s)\end{tabular}  & \begin{tabular}[c]{@{}c@{}}\method{} \\ (tokens/s)\end{tabular}  & \begin{tabular}[c]{@{}c@{}}Speedup \\ (\%)\end{tabular}  \\ \midrule
    \multirow{2}{*}{\begin{tabular}[c]{@{}c@{}}Mixtral \\ (8$\times$7B)\end{tabular}} & 32  &    14356.1 & 18357.4 &  \textbf{27.87}   \\
                \cline{2-5}  
                &   64& 25450.9 & 34044.7 & \textbf{33.77}  \\
             \bottomrule   
\end{tabular}		
		}}
	\end{center}
 \vspace{-6mm}
\end{table}

\textbf{Results on MoE trained from scratch}.
To validate the effectiveness of \method{} on large-scale datasets, we conducted training-from-scratch experiments on the popular MoE model~\cite{zhao2024longskywork,jiang2024mixtral}. These experiments spanned various data volumes and model sizes. Specifically, we trained two configurations of the Sky-MoE~\cite{zhao2024longskywork} with 8 experts: $8\times0.1$B (total parameter count of $0.5$B) and $8\times0.3$ (total parameter count of $2$B), using tokens from the RedPajama-v2 dataset~\cite{together2023redpajama} in sizes of 10B, 30B, and 300B. We report the training loss, which is equivalent to the validation loss in this context, as the model encounters each data point only once during training. In this experiment, we applied element-wise clipping to the estimated local gradient $\*{g}^n_k$ to reduce sensitivity to the compression hyperparameter $s$ in \method{}.

As shown in Table~\ref{tab:moeloss}, despite utilizing 4-bit gradient communication, \method{} achieved results consistent with full-precision Adam across different data volumes and model sizes. Unlike fine-tuning, training from scratch on large datasets better demonstrates the practical utility and communication efficiency of \method{}.

\textbf{Results on GPT2 trained from scratch}.  We train  GPT2-345M  on the OpenWebtext dataset~\cite{gokaslan2019openwebtext} of 52B tokens from scratch.  Fig.~\ref{fig:curve} (a) shows that 1) our 1-bit \method{}  has faster convergence speed than other 1-bit optimizers, and 2) our 4-bit   \method{}  even share similar behaviors as 16-bit Adam and is better than our 1-bit version.   So, without sacrificing performance, \method{} can improve communication efficiency, showing the superiority of our error-feedback strategy in maintaining training quality. 

\textbf{Results on LLAMA2 trained from scratch}. { To demonstrate the generality and versatility of \method{}, we integrate  \method{} with the SoTA Zero++, named \method{}-Zero++. This integration retains the communication efficiency of Zero++ while mitigating the information loss caused by 4-bit gradient quantization. Specifically, we train LLaMA2 models with 0.8B and 8B parameters from scratch using 30B and 5B tokens, respectively, sampled from the RedPajama-v2. Results  in Fig.~\ref{fig:curve} (b) and (c) show that in both cases, \method{}-Zero++ achieves training quality comparable to 16-bit AdamW, outperforming standalone Zero++. Notably, for the small-size 0.8B LLAMA2, \method{}-Zero++ demonstrates a significant improvement, highlighting its ability to effectively address the challenges of 4-bit gradient quantization. For the larger 8B model, where the model's stronger fitting capacity and the relatively small-size dataset reduce the loss gap between 4-bit Zero++ and 16-bit AdamW, \method{}-Zero++ still provides a measurable positive gain. It is worth noting that achieving lower losses for large-scale models under the same computational budget is inherently challenging~\cite{xie2024optimization}.
	
Additionally, \method{} introduces no extra computational overhead. For example, in the LLaMA2-8B experiments, AdamW required 26 hours of training on 32 NVIDIA A800 GPUs, whereas both Zero++ and \method{}-Zero++ completed training in just 21 hours, achieving a \emph{20\% speedup}. These findings further highlight the efficiency and practicality of \method{} for large-scale model training.}

\textbf{Results on LLAMA2-7B fine-tuned on downstream tasks}. 
We follow \cite{peng2023instruction} and fine-tune  LLAMA2-7B for three epochs on the alpaca-gpt4 dataset to evaluate commonsense reasoning ability. Here, we use the 4-bit gradient in 1-bit Adam, 0/1 Adam, since we find  that they are very unstable during training billion-scale models.  

Table~\ref{tab:lowbit_compare} reveals that \method{} outperforms all 4-bit optimizers, e.g., 0/1 Adam and Zero++, and even achieves comparable performance as 16-bit Adam, the official optimizer, on all commonsense reasoning tasks.  This well demonstrates the training quality of \method{}, which only uses a 4-bit gradient. Moreover,  other 4-bit optimizers often have much worse performance than  16-bit Adam. This is because their error-feedback indeed cannot well address the accumulated gradient quantization error over the iterations, while  \method{} introduces the moving average to stabilize the fluctuating quantization error and also restart the error to remove the impact from the out-of-date historical error. 

\textbf{Results on LLAMA2-7B trained with LoRA~\cite{hu2021lora}}.
To compare LoCO with other efficient communication methods, such as PowerSGD, without model sharding, we utilized the LoRA~\cite{hu2021lora} strategy for fine-tuning (as full-parameter fine-tuning under DDP mode would result in out-of-memory). We fine-tuned the LLaMA2-7B model on the Alpaca-GPT4 dataset. The results, presented in Table~\ref{tab:powersgd}, indicate that PowerSGD underperforms compared to LoCO and exhibits a  gap from the baseline. In contrast, LoCO achieves results comparable to 16-bit full-precision AdamW. Although PowerSGD can reduce communication overhead by adjusting the low-rank parameter $r$, its convergence is challenging to ensure. Additionally, PowerSGD lacks support for FSDP, leading to substantial memory overhead, making full-parameter fine-tuning impractical.

\vspace{-0.5em}
\subsection{Results on Training Speed}
\vspace{-0.3em}
Here, we investigate the training speed of \method{} by reporting its throughput (i.e., the number of consumed tokens per second) under different settings. For comprehensive investigation,  we test \method{} by using different model architectures,  node connections, and large-model training frameworks. We report the throughput of the popular  LLAMA2, Mistral, and Mixtral~(i.e., MoE-Mistral) on both the A100 cluster inter-connected with RoCE network and the A800 cluster inter-connected with Infiniband.  Due to limited space, we defer more training speed results in  Appendix~\ref{sec:moreresults}. \textit{For \method{}, we combine it with Adam for fairness.}

\textbf{Model architectures.} Table~\ref{tab:speedup} reveals that on all LLMs whose size varies from 7B to 70B,  \method{} makes a significant speedup on the official 16-bit Adam in terms of the throughput. Moreover, the larger the model, the greater \method{}  speeds up. For example,   \method{} achieves a speedup of 35.03\% for 70B LLAMA2   on  128 A100 GPUs and 42.24\% for 13B LLAMA2  on 128 A800 GPUs. This shows the good scalability of \method{}.

\textbf{GPU types.}   Table~\ref{tab:speedup} shows that the lower the bandwidth of a cluster, the more significant improvement \method{} can achieve. The A800 cluster has a lower bandwidth than the A100 cluster and shows a greater speedup.  For instance, on 7B Mistral, \method{} has a 22.49\% improvement on the A800 server but has a 14.61\% improvement on the A100 cluster. 

Moreover, the more GPUs in a cluster,  the more speedup   \method{} makes. For example, as shown in Table~\ref{tab:speedup},  on 13B LLAMA2,  the speedup of  \method{}  is improved from 24.58\% to 42.24\%  when the GPU number increases from 32 to 128 on the A800 server.  This is because if a server's bandwidth becomes smaller or its GPU number increases, its communication cost (including volume and round) will increase and become the training bottleneck, leading to slow training speed. For these cases, \method{} can greatly reduce communication costs by using the low-bit gradient, thus significantly improving the training speed.

\textbf{Large-scale training frameworks.}  Table~\ref{tab:speedup} evaluates \method{} via  Megatron-LM training framework~\cite{shoeybi2019megatron}, while Table~\ref{tab:fsdp} focuses the speed on PyTorch FSDP  framework~\cite{zhao2023pytorch}. Megatron-LM uses comprehensive parallelizations to improve training efficiency, e.g., data, pipeline, and tensor parallelism, and is widely used in LLMs training, while   FSDP enhances efficiency by partitioning model, gradients, and optimizer states,  and improves the communication efficiency during back-propagation.  
Table~\ref{tab:speedup} and  Table~\ref{tab:fsdp} show that on both frameworks,  \method{} makes notable speed up and shows its high compatibility.

\textbf{Peak memory comparison.} 
Table~\ref{tab:mem} shows that   \method{} often requires only an additional 9-10\%  memory overhead compared with the official 16-bit Adam for both LLAMA2 and Mixtral. For practical LLM training, GPU memory is often not fully used since selecting a widely used and proper maximum token length to use GPU memory fully is hard, making some extra GPU memory available. In this way,  the remaining available GPU memory can be used for the extra 10\% memory cost in \method{}. Moreover, as shown in Table~\ref{tab:speedup}, \method{} often brings 15\% to 40\% overall training speedup.

\subsection{Ablation Experiments}
We delve into the effects of various components of \method{}, including 1) error-feedback,  2) moving averaging on error, 3) error compression, and 4) error reset.   We follow Sec.~\ref{model_cfg} to fine-tune  Mixtral.  Table~\ref{tab:abla_study} reports the results of \method{} with Adam as its optimizer.


\textbf{Error-feedback.}  By comparing \method{}1 and \method{}2 in Table~\ref{tab:abla_study}, one can observe that incorporating error-feedback directly almost does not bring improvement.  
Specifically, integrating error-feedback slightly impairs the performance on the coding-related benchmarks, e.g., HumanE and MBPP. This may be attributed to the discontinuity in compression, which results in significant variance for compression error. Hence, we still need other components of \method{} to boost the performance jointly. 

\textbf{Moving average on error.} To mitigate instability in vanilla error-feedback which solely uses compression error from one previous iteration, we design a moving average on all historical compression errors to estimate more stable and accurate compression error.  
  \method{}2 and \method{}3 in Table~\ref{tab:abla_study} shows that moving average on error improves a lot on downstream tasks, particularly in coding-related benchmarks. 

\textbf{Error compression.} To save  GPU memory, we compress high-precision compensation errors into 8-bit ones. By comparing \method{}4 and \method{}5 in Table~\ref{tab:abla_study}, this compression only brings negligible performance degradation while further reducing the memory footprint of \method{}. This enhances the applicability of \method{} for large-model training.

\textbf{Error reset.} Along with the training, the very early compensation errors  become outdated and is not suitable for current estimation. So we design an error reset mechanism, a critical element in both theoretical and practical realms. Our theoretical analysis, particularly in the proofs of Theorems~\ref{thm:sgd}, highlights the significance of periodic error resets in controlling the error scale and ensuring algorithmic convergence. In practice,   By comparing \method{}5 and \method{}6 in Table~\ref{tab:abla_study}, setting the error reset frequency ($T_c$) to either 128 or 512 has shown notable performance boosts. 
This error resetting strategy ensures an accurate estimation, and aligns with our theoretical findings, thereby enhancing the effectiveness of \method{}. 
Notably, for simplicity, we always set the reset frequency as 512 in all other experiments.

\section{Conclusion}

\method{} addresses the challenges of efficient large-model training with low-precision gradient communication. It successfully compensates gradients before compression, ensuring effective communication without sacrificing training quality. Distinguished by its low computational and memory requirements, \method{} advances beyond traditional compression methods by preventing error accumulation during the optimization process. Its compatibility with various optimizers and gradient partitioning techniques in advanced training frameworks demonstrates its versatility and practical utility.

\ifCLASSOPTIONcaptionsoff
  \newpage
\fi
\bibliographystyle{IEEEtran}
\bibliography{loco}

\begin{IEEEbiography}[{\includegraphics[width=1in,height=1.25in,clip,keepaspectratio]{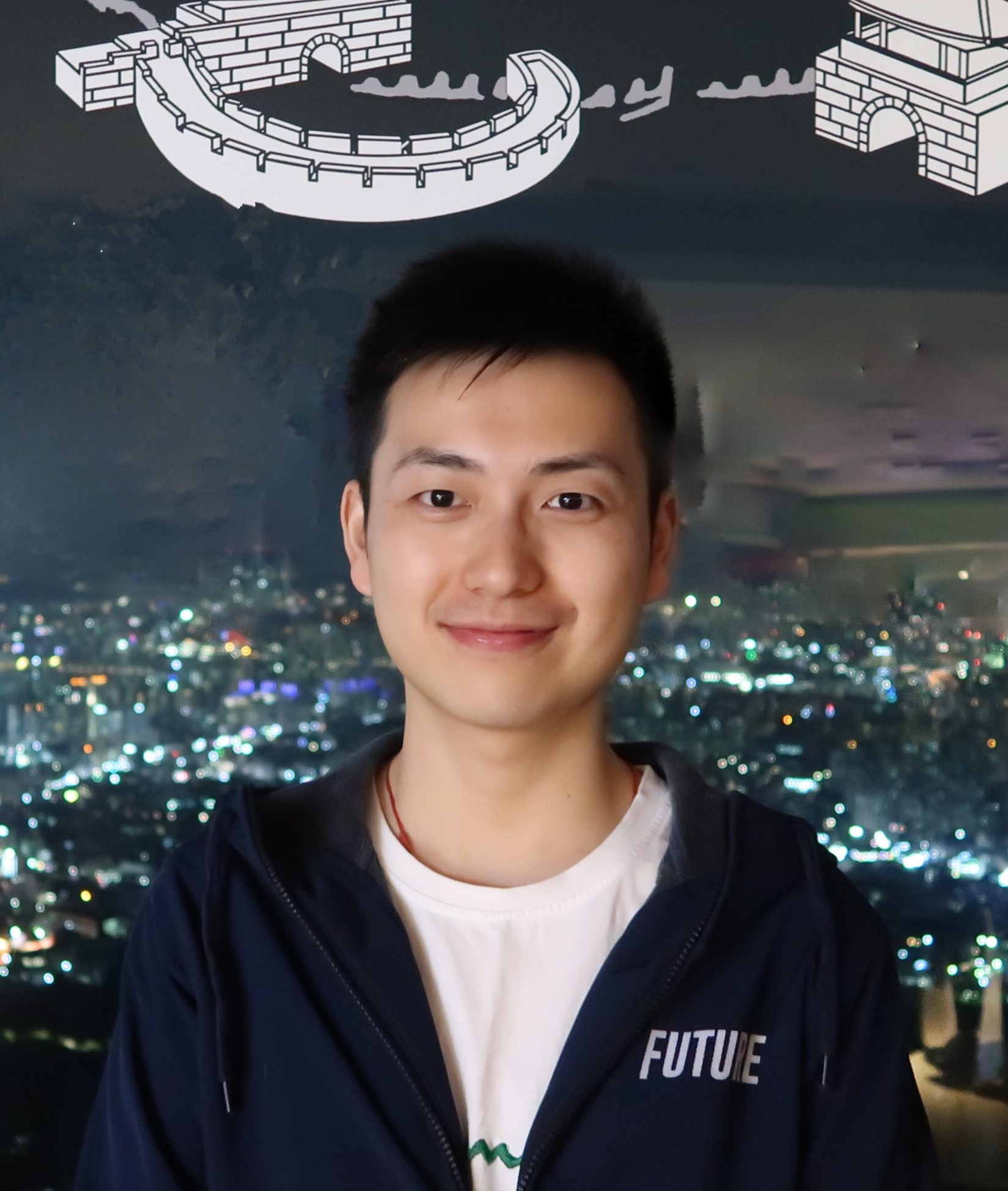}}]{Xingyu Xie}
received his Ph.D. degree from Peking University in 2023. He is currently a Research Fellow at the Department of Mathematics, National University of Singapore. His current research interests include large-scale optimization and deep learning.
\end{IEEEbiography}

\begin{IEEEbiography}[{\includegraphics[width=1in,height=1.25in,clip,keepaspectratio]{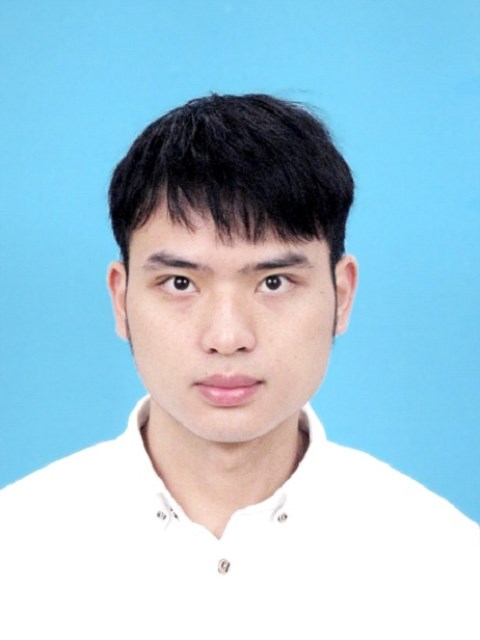}}]{Zhijie Lin} received the B.E. degree, in 2019 and obtained master’s degree, in 2022, in computer science and technology from Zhejiang University, China. Now he is a research engineer at TikTok Research, Singapore. Before he also worked as a research engineer at Sea AI Lab, Singapore. His research interests include computer vision and generation models.
\end{IEEEbiography}

\begin{IEEEbiography}[{\includegraphics[width=1.05in,height=1.3in,clip]{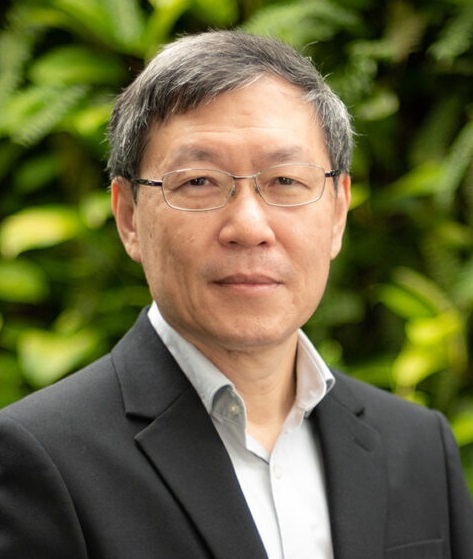}}]{Kim-Chuan Toh}
obtained his PhD from Cornell University in 1996.
Currently, he is a Professor in the Department of Mathematics at 
National University of Singapore. His research interests are on algorithms for convex conic optimization, and
optimization problems arising from machine learning and statistics.
He received the INFORMS Optimization Society Farkas Prize in 2017 and the Mathematical Optimization Society Beale-Orchard Hays Prize in 2018.
He is a Fellow of the Society for Industrial and Applied Mathematics.
\end{IEEEbiography}

\begin{IEEEbiography}[{\includegraphics[width=1in,height=1.25in,clip,keepaspectratio]{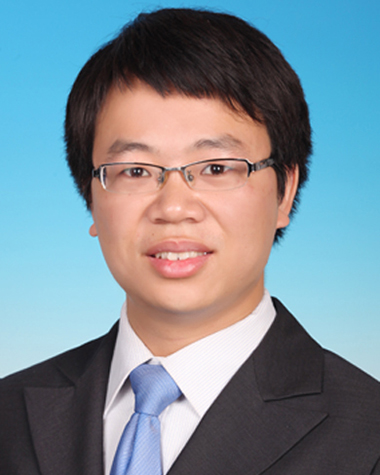}}]{Pan Zhou} received Master Degree at  Peking University in 2016 and obtained Ph.D. Degree at National University of Singapore in 2019. Now
he is an assistant professor at Singapore Management University, Singapore. Before he also worked as a research scientist at Salesforce and Sea AI Lab, Singapore. His research interests include computer vision, machine learning, and optimization. He was the winner of the Microsoft Research Asia Fellowship 2018.
\end{IEEEbiography}

\clearpage
\onecolumn
\appendices
\begin{center}
\text{\huge{\textbf{\method{}: Low-Bit Communication Adaptor for}}
}
\\
\vspace{1mm}
\text{\huge{\textbf{Large-scale Model Training}}
}
\\
\vspace{1mm}
\text{\huge{(Supplementary Material)}
}
\end{center}
The appendix supplements the paper titled ``Low-Bit Communication Adaptor for Large-scale Model Training" with additional experimental results and technical proofs of convergence. It is structured as follows for ease of navigation and comprehension:

Appendix~\ref{sec:background} provides an introduction to concepts and background knowledge related to distributed training of large-scale models. This section covers sharding strategies for large models and explains operations such as reduce-scatter and all-to-all. These concepts are foundational to understanding the distributed training environment in which \method{} operates.

Appendix~\ref{eval_details} details the specific model configurations and training parameters used in the LLM experiments presented in the main text. Additionally, it includes additional experimental results, offering a more comprehensive view of \method{}'s performance. This section features detailed comparisons of speedup ratios under various settings, models, and training frameworks, providing a deeper insight into the adaptability and efficiency of \method{}.

Appendix~\ref{sec:mainproof} begins by establishing several properties of \method{}, followed by proofs demonstrating the convergence rates of SGD and Adam-type optimizers when combined with the \method{} strategy. These proofs are vital for validating the theoretical underpinnings of \method{} and its effectiveness in optimizing large-scale model training.

Appendix~\ref{proofofAuxiliary} contains detailed proofs of auxiliary lemmas and properties that support the main arguments and findings in Appendix \ref{sec:mainproof}. 

\section{Preliminary of Distributed Communication}\label{sec:background}
We introduce several concepts relative to the modern  distributed communication system in this section

\subsection{Scattering and Gathering}\label{sec:all2all}
All-reduce, reduce-scatter, and all-gather, shown in Fig.~\ref{fig:reduce}, are key operations in distributed computing, particularly in the context of LLM training, where they are used to aggregate data like gradients across multiple processors or GPUs.

 \subsubsection{All-reduce} All-reduce is a collective operation where data from all processors (like gradients from different GPUs) is combined and then redistributed to each processor. This means every processor ends up with the same, fully aggregated result.

\subsubsection{Reduce-scatter} Reduce-scatter is the first phase of the all-reduce operation. In this step, each processor contributes its data, which is then partially combined and scattered back to the processors. Thus, each processor ends up with a fragment of the total aggregated data. This process involves sequential data sending, receiving, and reducing operations in a ring-like fashion, ensuring each processor receives a portion of the final aggregated result.

\subsubsection{All-gather} All-gather follows reduce-scatter in the all-reduce process. During this phase, each processor shares its fragment of aggregated data with every other processor. By the end of all-gather, all processors have the complete set of aggregated data, vital for further computations.
The training process for LLMs utilizing these operations typically unfolds as follows: after the back-propagation phase, a reduce-scatter is employed to distribute and partially aggregate the gradients across the different nodes. Each node then uses these reduced gradients to update its own model weight partition and optimizer's state. 
Before the next forward propagation begins, the all-gather operation is utilized to synchronize the weights of the model across all nodes. This ensures that each GPU starts the next iteration of training with the same, updated model parameters. This sequence of operations not only enhances the training efficiency but also ensures consistency and scalability in the distributed training of large-scale models.

\subsubsection{Ring-based Reduce-scatter vs. All-to-all}
As shown in Fig.~\ref{fig:shard}. The ring-based reduce-scatter operation efficiently divides and distributes data across multiple devices in a ring-like configuration. Each device receives a data chunk, performs a reduction operation (like summing), and passes the reduced data to the next device. This cycle continues until every device has a portion of the aggregated data. The key advantage of this approach is its efficient use of cluster bandwidth and balanced workload distribution among all devices in the ring.

On the other hand, the All-to-all operation, often chosen for its specific benefits like avoiding overflow in certain contexts, involves each process in the cluster sending and receiving unique data segments to and from every other process. It is particularly effective in scenarios where each node needs to have a complete picture of the data distributed across the cluster network. The total communication volume in an alltoall operation can be similar to that of a ring-based reduce-scatter when appropriately implemented. This makes alltoall a viable and sometimes preferred choice in certain distributed computing tasks, such as scenarios where overflow avoidance is critical.

\begin{figure}[t]
\centering
\includegraphics[width=\linewidth]{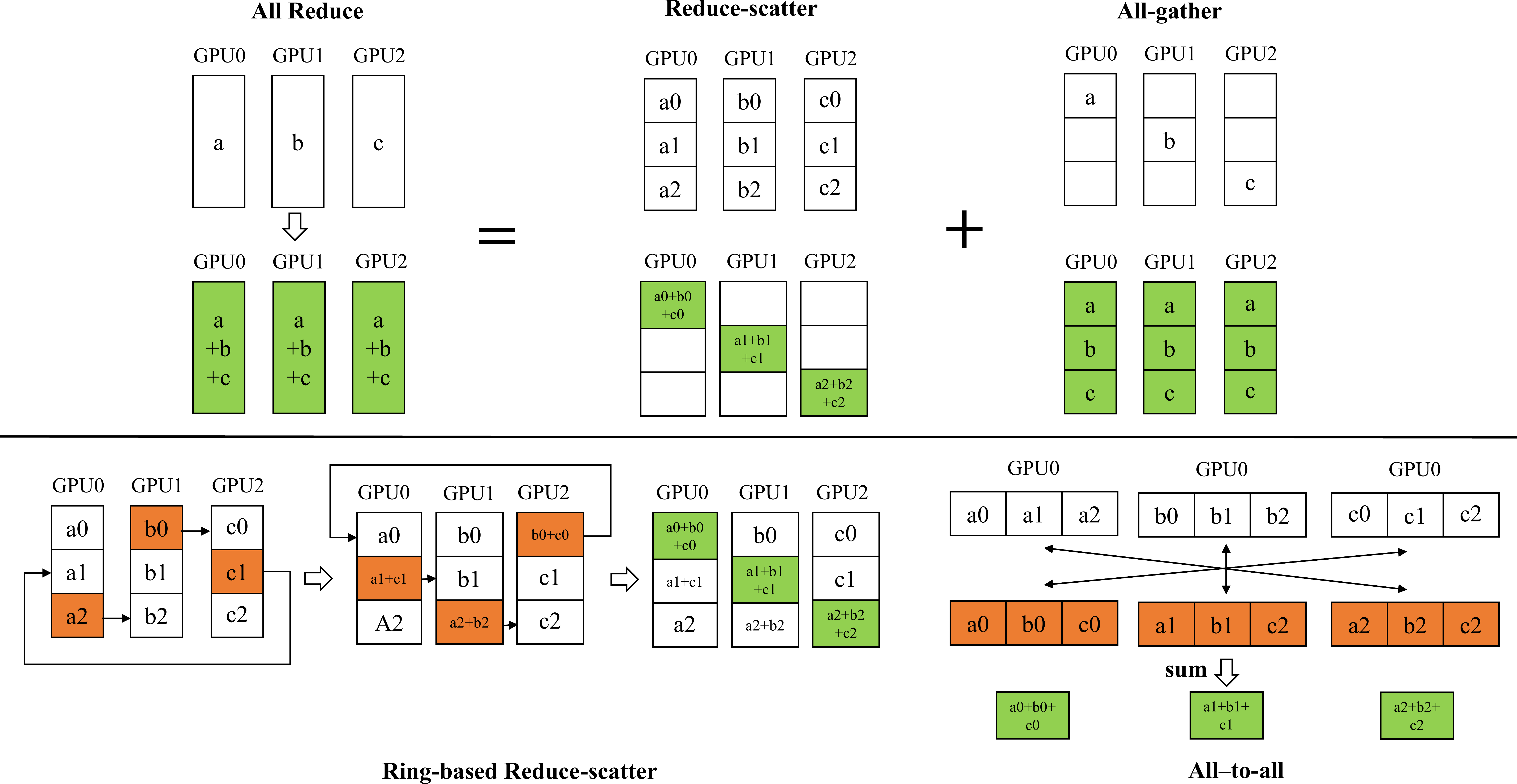}		\caption{Overview of All-Reduce and Its Component Operations. The All-Reduce process is depicted in two main phases. Initially, the Reduce-Scatter operations is performed, where gradients are divided and summed up in equal blocks across GPUs according to their ranks. This is followed by the All-Gather phase, where each GPU shares its segment of the aggregated gradients, ensuring the complete set of gradients is available to all GPUs. Additionally, the figure includes representations of the Ring reduce-scatter and Alltoall operations, integral in the gradient distribution and aggregation process across the cluster.}\label{fig:reduce}
\end{figure}

\subsection{Modern Sharding Strategy}\label{sec:fsdp}
Fully Sharded Data Parallelism (FSDP) revolutionizes the training of deep learning models by optimizing memory usage and computational efficiency across multiple GPUs. As shown in Fig.~\ref{fig:shard}, here is an introduction to FSDP.

\subsubsection{Initial Framework without Sharding} Traditionally, in data parallelism without sharding, every GPU holds a complete set of the model's weights, gradients, and optimizer states. The communication between GPUs is limited to performing an all-reduce operation on the entire gradients to facilitate model updates, eliminating the need for sharing other variables. This straightforward communication pattern simplifies operations but comes at the expense of high memory consumption.

\subsubsection{Sharded Model} Implementing a sharding strategy significantly enhances memory efficiency. Under this strategy, each GPU maintains only a local partition of the optimizer states, avoiding the need to communicate these states between GPUs. Moreover, each GPU keeps only a partition of the averaged gradients. After backpropagation, non-local portions of the gradients are sent to the respective GPUs via reduce-scatter and subsequently released from memory. This ensures that each GPU updates its segment of the optimizer states and a fraction of the weights. An all-gather operation on the weights follows, maintaining consistency across all GPUs and enabling partial model updates with a reduced memory footprint.

\subsubsection{Fully Sharded Data Parallelism} Advancing further, FSDP limits each GPU to holding just a partition of the model weights, eliminating the presence of the complete model on any single GPU. Prior to forward propagation, GPUs collect necessary weight partitions from each other to assemble a full model in memory, conduct forward and backward propagation, and then proceed with gradient reduce-scatter. Post-backpropagation, any non-local weights and gradients are discarded to conserve memory, significantly lowering the memory requirements and fostering complex model training that was previously infeasible due to memory constraints.

While the above explanation outlines the logical framework of FSDP for ease of understanding, it's crucial to acknowledge the sophisticated optimizations embedded within its architecture aimed at further reducing peak memory usage. For more details, please refer to the Zero strategy~\cite{rajbhandari2020zero} and Pytorch FSDP~\cite{zhao2023pytorch}. These include, but are not limited to, layer-wise gradient communication and the strategic release of memory during the backpropagation phase. Such detailed optimizations, though not elaborated here, are fundamental to FSDP's effectiveness in resource management, enabling the scalable and efficient training of large-scale models by adeptly minimizing the peak memory footprint.

\begin{figure}[t]
\centering
\includegraphics[width=0.8\linewidth]{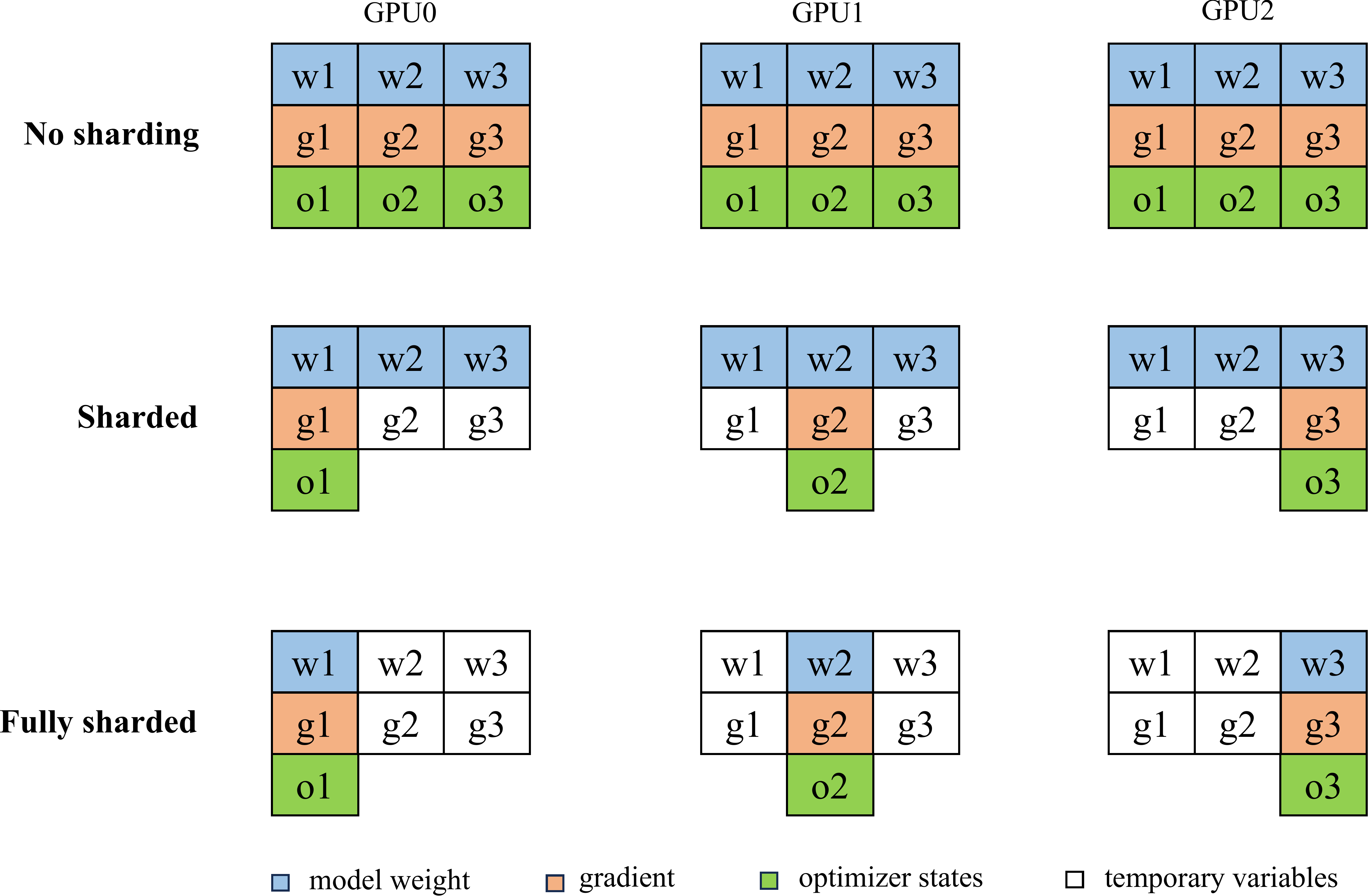}		\caption{Overview of sharding strategy. (1) Initial Framework without Sharding, where each GPU stores the entire model, leading to high memory usage but simplified communication; (2) Sharded Model, introducing a strategy where GPUs maintain only local partitions of optimizer states and averaged gradients, enhancing memory efficiency through reduce-scatter and all-gather operations for gradient and weight management; (3) Fully Sharded Data Parallelism, further optimizing memory by limiting each GPU to a partition of the model weights, necessitating inter-GPU collection of weight partitions for model assembly, significantly reducing memory requirements and enabling the training of large-scale models previously constrained by memory limitations.}\label{fig:shard}
\end{figure}

\begin{table*}[t!]
	\begin{center}
		\caption{Detailed end-to-end speedup of  \method{} in prevalent Megatron-LM framework for efficient multi-node LLMs training with different model sizes, gradient accumulation numer, GPU number, and the node connection.}
		\label{tab:speedup_full}
		\setlength{\tabcolsep}{5pt} 
		\renewcommand{\arraystretch}{3.2}
		{ \fontsize{8.5}{3.5}\selectfont{
				\begin{tabular}{l|c|ccc|ccc|ccc}
					\toprule
				&	& \multicolumn{3}{c|}{32$\times$ NVIDIA A100 (RoCE v2)} & \multicolumn{3}{c|}{64$\times$ NVIDIA A100 (RoCE v2)}  & \multicolumn{3}{c}{128$\times$ NVIDIA A100 (RoCE v2)}     \\ \cline{3-11} 
					Model       &  \begin{tabular}[c]{@{}c@{}}Accum. \\ Num.\end{tabular} & \begin{tabular}[c]{@{}c@{}}Baseline \\ (tokens/s)\end{tabular}  & \begin{tabular}[c]{@{}c@{}}\method{} \\ (tokens/s)\end{tabular}  & \begin{tabular}[c]{@{}c@{}}Speedup \\ (\%)\end{tabular}    & \begin{tabular}[c]{@{}c@{}}Baseline \\ (tokens/s)\end{tabular}  & \begin{tabular}[c]{@{}c@{}}\method{} \\ (tokens/s)\end{tabular}  & \begin{tabular}[c]{@{}c@{}}Speedup \\ (\%)\end{tabular}   & \begin{tabular}[c]{@{}c@{}}Baseline \\ (tokens/s)\end{tabular}  & \begin{tabular}[c]{@{}c@{}}\method{} \\ (tokens/s)\end{tabular}  & \begin{tabular}[c]{@{}c@{}}Speedup \\ (\%)\end{tabular}   \\ \midrule
				\multirow{3}{*}{\begin{tabular}[c]{@{}c@{}}LLAMA2 \\ (7B)\end{tabular}}  &  4 & 75544.9 & 78911.7 & 4.47  &  148071.9 & 156369.9 & 5.60 &  284840.8 & 307657.4 & 8.01 \\
			&  2 & 68330.6 & 73706.1 & 7.87  &  131484.3 & 145277.7 & 10.49   &  254703.8 & 284862.9 & 11.84\\
                &  1 & 57230.2 & 65376.3&  \textbf{14.23} &  108680.5 & 127263.1 & \textbf{17.10} &  212373.9 & 251701.9 & \textbf{18.50}   \\
    \midrule
\multirow{3}{*}{\begin{tabular}[c]{@{}c@{}}Mistral  \\ (7B)\end{tabular}}  &  4 & 74354.6&78674.1&5.81  &  145855.5& 154816.9 & 6.14 & 284082.2 & 305136.9 & 7.41\\
			&  2 & 65345.6 & 72734.2  &  11.31  & 128964.8&144120.13 & 11.75   & 249414.7 & 281070.5 & 12.69\\
                &  1 & 55947.3 & 64123.7 &\textbf{14.61} &  105198.2 & 125422.7 & \textbf{19.22} &   206053.7 & 247468.3 & \textbf{20.10}   \\ \midrule 
\multirow{3}{*}{\begin{tabular}[c]{@{}c@{}}LLAMA2 \\ (13B)\end{tabular}}  &  4 & 40341.8 &43092.1&6.82  &  71847.3& 79106.9 & 10.10 & 139677.0 & 156768.8 & 12.23 \\
			&  2 & 35972.6 & 40097.4  &  11.47  & 58235.9& 69345.9 & 19.07   &   113070.9 &  136932.6 & 21.10\\
                &  1 & 30555.9 & 35683.2 &\textbf{16.78} &   43941.6 & 55322.9 & \textbf{25.90} &  83160.2& 108577.2  & \textbf{30.56}   \\  
                \midrule
\multirow{3}{*}{\begin{tabular}[c]{@{}c@{}}LLAMA2 \\ (70B)\end{tabular}}  &  4 & \multicolumn{3}{c}{\multirow{3}{*}{N/A, since the Data Parallel is $\bm{1}$}} &  8108.3& 9870.0 & 21.73 & 15938.6 & 19612.1 & 23.05 \\
			&  2 & &&  & 5110.6& 6503.7 & 27.26   &   9619.7 & 12387.2 & 28.77\\
                &  1 & && &   2869.2 & 3803.2 & \textbf{32.55} &  5263.6& 7107.6  & \textbf{35.03}   \\  
                \bottomrule
					&	& \multicolumn{3}{c|}{32$\times$ NVIDIA A800 (Infiniband)} & \multicolumn{3}{c|}{64$\times$ NVIDIA A800 (Infiniband)}  & \multicolumn{3}{c}{128$\times$ NVIDIA A800 (Infiniband)}     \\ \cline{3-11}
					       &  \begin{tabular}[c]{@{}c@{}}Accum. \\ Num.\end{tabular} & \begin{tabular}[c]{@{}c@{}}Baseline \\ (tokens/s)\end{tabular}  & \begin{tabular}[c]{@{}c@{}}\method{} \\ (tokens/s)\end{tabular}  & \begin{tabular}[c]{@{}c@{}}Speedup \\ (\%)\end{tabular}    & \begin{tabular}[c]{@{}c@{}}Baseline \\ (tokens/s)\end{tabular}  & \begin{tabular}[c]{@{}c@{}}\method{} \\ (tokens/s)\end{tabular}  & \begin{tabular}[c]{@{}c@{}}Speedup \\ (\%)\end{tabular}   & \begin{tabular}[c]{@{}c@{}}Baseline \\ (tokens/s)\end{tabular}  & \begin{tabular}[c]{@{}c@{}}\method{} \\ (tokens/s)\end{tabular}  & \begin{tabular}[c]{@{}c@{}}Speedup \\ (\%)\end{tabular}   \\ \midrule
				\multirow{3}{*}{\begin{tabular}[c]{@{}c@{}}LLAMA2 \\ (7B)\end{tabular}}  &  4 & 73047.8 & 77834.2 & 6.55  &  136605.5 & 151714.2 & 11.06 &  264459.1 & 295077.9   &11.58 \\
			&  2 & 65542.2 & 73312.9 & 11.86  &  116276.3 & 139874.8 & 20.30   & 216842.1 & 265101.3 & 22.26\\
                &  1 & 54186.8 & 65862.1&  \textbf{21.55} &  89555.4& 120625.6 & \textbf{34.69} &  161447.6 & 224887.7 & \textbf{39.29}   \\
    \midrule
    \multirow{3}{*}{\begin{tabular}[c]{@{}c@{}}Mistral \\ (7B)\end{tabular}}  &  4 & 71150.4 &76262.5 & 7.18  &  132480.4 & 147806.4 & 11.57 &  254865.7 & 285780.9   & 12.13 \\
			&  2 & 63195.6 & 71579.4 & 13.27  &  111917.1 & 135508.3 & 21.08  & 209780.7 & 258785.6 & 23.36\\
                &  1 & 51896.8 & 63568.5&  \textbf{22.49} &  85334.5& 115355.6 & \textbf{35.18} &  155308.7 & 217494.4 & \textbf{40.04}   \\
    \midrule
\multirow{3}{*}{\begin{tabular}[c]{@{}c@{}}LLAMA2 \\ (13B)\end{tabular}}  &  4 & 42515.2 &46195.4&8.65  &  79554.6& 89581.0 & 12.60 & 151598.8 & 173761.8 & 14.62 \\
			&  2 & 37922.1 & 43062.3  &  13.55  & 66455.2& 81644.0 & 22.86   &  124160.3 & 155571.1 & 25.30\\
                &  1 & 30682.9 & 38226.1 &\textbf{24.58} &  49907.4 & 69409.0 & \textbf{39.08} &  90446.3 & 128649.6  & \textbf{42.24}   \\
                \bottomrule 
\end{tabular}		
		}}
	\end{center}
\end{table*}

\section{Experimental Details}\label{eval_details}

\begin{table*}[t!]
	\begin{center}
		\caption{Speedup achieved by our method against  16-bit communication in PyTorch FSDP framework on MoE model.}
		\label{tab:fsdp_full}
		\setlength{\tabcolsep}{6pt} 
		\renewcommand{\arraystretch}{3.5}
		{ \fontsize{8.5}{3.5}\selectfont{
				\begin{tabular}{l|c|c|ccc|ccc}
					\toprule
			&	&	& \multicolumn{3}{c|}{32$\times$ NVIDIA A800 (Infiniband)} & \multicolumn{3}{c}{64$\times$ NVIDIA A800 (Infiniband)}       \\ \cline{4-9} 
					Model&  \begin{tabular}[c]{@{}c@{}}Param. \\  Sharded\end{tabular}       &  \begin{tabular}[c]{@{}c@{}}Accum. \\ Num.\end{tabular} & \begin{tabular}[c]{@{}c@{}}Baseline \\ (tokens/s)\end{tabular}  & \begin{tabular}[c]{@{}c@{}}\method{} \\ (tokens/s)\end{tabular}  & \begin{tabular}[c]{@{}c@{}}Speedup \\ (\%)\end{tabular}    & \begin{tabular}[c]{@{}c@{}}Baseline \\ (tokens/s)\end{tabular}  & \begin{tabular}[c]{@{}c@{}}\method{} \\ (tokens/s)\end{tabular}  & \begin{tabular}[c]{@{}c@{}}Speedup \\ (\%)\end{tabular}  \\ \midrule
    \multirow{3}{*}{\begin{tabular}[c]{@{}c@{}}Mixtral \\ (8$\times$7B)\end{tabular}}  & \multirow{3}{*}{True} &  4 & 76204.6 & 85250.1 & 11.87  &  135825.9 & 148523.5 & 9.35  \\
			& &  2 & 34813.2 & 40329.8 & 15.85  &  60963.7 & 71820.3 & 17.81 \\
                & &  1 & 14356.1 & 18357.4 &  \textbf{27.87} &  25450.9 & 34044.7 & \textbf{33.77}  \\
             \bottomrule   
\end{tabular}		
		}}
	\end{center}
\end{table*} 

\subsection{Additional Experimental Results}\label{sec:moreresults}
Here, we present the additional result for the speedup of \method{} on various settings. Firstly, we introduce the concept of accumulation numbers.

\textbf{Accumulation number:} The accumulation number refers to the number of forward and backward passes accumulated before a parameter update occurs in a machine learning model, especially in large-scale training scenarios. In the context of large model training, accumulation number is crucial due to memory constraints and computational efficiency considerations. By accumulating gradients over multiple iterations before updating the model parameters, it's possible to effectively train large models with limited memory resources, especially the large batch size setting. A lower accumulation number leads to a more pronounced speedup in training. This is primarily because a lower accumulation number increases the frequency of communication between nodes. In large-scale training setups, especially with thousands of GPUs, the accumulation number is typically low (e.g., 1 or 2). This low accumulation number significantly magnifies the benefits of methods like \method{}, which reduce communication overhead.

\textbf{4D parallelism:} \method{}'s effect is even more pronounced when considering 4D parallel strategies, i.e., data parallelism, pipeline parallelism, tensor parallelism, and expert parallelism. In scenarios where a large number of GPUs are employed, data parallelism often emerges as the simplest and most effective method to increase training parallelism. However, tensor and pipeline parallelism require careful consideration of intra-node and inter-node communications, where the cost of all-reduce operations between nodes can be high. The number of experts available constrains expert parallelism's efficiency. \method{} greatly reduces communication volumes in the context of data parallelism, accelerating training speed significantly.

\textbf{Results on training speed:}
Table~\ref{tab:speedup_full} and Table~\ref{tab:fsdp_full} show the detailed training speed of \method{}. Notably, there's a significant improvement in throughput across different models and scales, with larger models seeing more pronounced acceleration. For instance, the 70B LLAMA2 model on the A100 cluster saw a speedup of 23.05\% to 35.03\%, and the 13B LLAMA2 model on the A800 cluster achieved from 14.62\% to 42.24\% speedup.

When considering node connections, \method{} demonstrates more substantial improvements in lower bandwidth environments (like the A800 cluster) compared to higher bandwidth scenarios (like the A100 cluster). This suggests that \method{} is particularly effective in optimizing communication efficiency, making it a vital tool in large model training.

Finally, when assessing large-scale training frameworks such as Megatron-LM and PyTorch FSDP, \method{} consistently shows a significant increase in token throughput speed. This highlights its efficacy across different frameworks and training environments, further establishing its utility in large-model training scenarios.

\subsection{Model Configurations}\label{model_cfg}
The following are the detailed experiment settings for the main parts. All experiments are performed with bfloat-16 precision in default.

\textbf{GPT-2:} Utilizing guidance from Megatron-DeepSpeed\footnote{https://github.com/microsoft/Megatron-DeepSpeed}, as detailed by DeepSpeed~\cite{rasley2020deepspeed}, we undertook the training of the GPT2-345M model from scratch. Our data source was the OpenWebtext dataset~\cite{gokaslan2019openwebtext}, through which we processed a substantial total of 52B tokens. The experiment was configured with a zero optimization stage set to 2, deliberately excluding the use of model and pipeline parallelism to simplify our setup. Our global batch size was set at 512, with a learning rate of 3.0e-4. In managing the training dynamics, we applied a global gradient norm clipping of 1 and opted for Adam's optimizer settings with $\beta_1=0.9$ and $\beta_2=0.95$.

\textbf{LLAMA2-7B:}
For LLAMA2-7B, we conduct two experiments on it.
\textbf{1)} Within the Fully Sharded Data Parallel (FSDP) framework, and following protocols from LLAMA-Accessory~\cite{zhang2023llamaadapter} and the GPT-4-LLM repository\footnote{https://github.com/Instruction-Tuning-with-GPT-4/GPT-4-LLM}, we conducted a supervised fine-tuning spanning three epochs on the alpaca-gpt4 dataset as per \cite{peng2023instruction}. This process included partitioning both gradient and optimizer states across GPUs, with a global batch size of 64, a learning rate of 2e-5, gradient clipping at 2, and model parallelism set to 2.
    
\textbf{2)} Utilizing the Megatron-LM framework~\cite{shoeybi2019megatron} and following the methodology of Tiny-LLAMA~\cite{zhang2024tinyllama}, we performed fine-tuning on the sampled RedPajama dataset\footnote{https://huggingface.co/datasets/togethercomputer/RedPajama-Data-1T-Sample}, consuming 8 billion tokens. This setup involved a model parallelism of 8, a total batch size of 1024, and a peak learning rate of 1e-5, avoiding pipeline parallelism. We matched Adam optimizer's parameters to those recommended by epfLLM~\cite{epfmgtrn}, setting $\beta_1=0.9$, $\beta_2=0.95$, and an $\epsilon$ value of 1e-5, with a maximum token length of 4096 and an $\epsilon$ for layer norm also at 1e-5.

\textbf{LLAMA2-13B:} The fine-tuning of LLAMA2-13B adhered to the identical training hyperparameters as those established for LLAMA2-7B. It's important to note that we retained most model-related hyperparameters from the Meta's officially released configurations.

\textbf{LLAMA2-70B:} To accommodate the large-scale model within our memory-constrained training cluster, we made necessary adjustments to the parallelization parameters. Specifically, we set the model parallelism to 8 and the pipeline parallelism to 4, with a max learning rate carefully reduced to 1.5e-4. All other parameters were consistently aligned with those of the LLAMA2-7B and LLAMA2-13B models.

\textbf{Mistral-7B:} The training configuration for Mistral-7B closely mirrored that of the LLAMA2-7B and LLAMA2-13B models, with the sole exception being a larger maximum learning rate of 5e-5.
    
\textbf{Mixtral 8$\times$7B:}
Drawing from the LLAMA-Accessory~\cite{zhang2023llamaadapter}, our investigation into the MoE model spanned two experiments, with settings largely coming from LLAMA-Accessory's demonstrations. We leveraged the FSDP's checkpointing strategy to alleviate memory demands, setting expert parallelism at 8.

\textbf{1)}  For the outcomes presented in Table~\ref{tab:eval}, Mixtral was fine-tuned on the ultrachat-200k dataset~\cite{ding2023enhancing} over a single epoch. Here, the maximum learning rate was 2.5e-6, with a global batch size of 128 and a maximum token length of 4096. The Pytorch FSDP sharding strategy was employed to manage gradient and optimizer states efficiently.

\textbf{2)} For the ablation study, we utilized the code-alpaca-v1 dataset~\cite{luo2023wizardcoder} for one epoch of fine-tuning, applying a fully sharded strategy to divide gradient, optimizer states, and model weight. Parameters were set with a maximum learning rate of 3e-6, a global batch size of 128, a maximum token length of 2048, and gradient clipping at 2.

\subsection{Benchmarks}\label{benchmarks}
Following prior works~\cite{jiang2024mixtral}, we conduct downstream evaluation on the popular platform Opencompass~\cite{2023opencompass}, and the wide variety of benchmarks used for evaluation are categorized as follows:
\begin{itemize}
    \item \textbf{Commonsense Reasoning~(0-shot):} Hellaswag~\cite{zellers2019hellaswag}, Winogrande~\cite{sakaguchi2021winogrande}, PIQA~\cite{bisk2020piqa}, ARC-Easy and ARC-Challenge~\cite{clark2018think}.
    \item \textbf{World Knowledge~(5-shot):} NaturalQuestions~\cite{kwiatkowski2019natural} and TriviaQA~\cite{joshi2017triviaqa}.
    \item \textbf{Math:} GSM8K~(8-shot)~\cite{cobbe2021training} and MATH~(4-shot)~\cite{hendrycks2021measuring}.
    \item \textbf{Code}: Humaneval~(0-shot)~\cite{chen2021evaluating} and MBPP~(3-shot)~\cite{austin2021program}.
    \item \textbf{Popular aggregated results:} MMLU~(5-shot)~\cite{hendrycks2020measuring}.
\end{itemize}

\section{Detailed Theoretical Analysis}\label{sec:mainproof}
\subsection{Auxiliary Lemmas}
\begin{lemma}\label{gradientrlation}
For the gradient 	$\tilde{\*g}_{k}  =   \frac{1}{N}\sum\nolimits_{n=1}^N \operatorname{decompressor}(\tilde{\*h}^n_{k+1})$ used to update the parameter in Algorithm~\ref{alg:BitCom}, we have:
\begin{equation*}
	\begin{aligned}
			\tilde{\*g}_k & =    \frac{1}{N}\sum_{n=1}^N  \left( {\*g}^n_k +  \frac{ 1}{\beta} \left( \emti{k}{n} - \emti{k+1}{n} \right)  +   \operatorname{decompressor}( \operatorname{compressor} (\emti{k}{n}))  -    \emti{k}{n}  \right)\\
   & = \*g_k + \frac{1}{\beta} \qty(\emti{k}{} - \emti{k+1}{}) + \emhi{k}{} - \emti{k}{},
	\end{aligned}
\end{equation*}
where we let $ \*g_k \coloneqq \frac{1}{N}\sum_{n=1}^N    \*g^n_k$,  $\emti{k}{}\coloneqq  \frac{1}{N}\sum_{n=1}^N   \emti{k}{n}$,  $\emhi{k}{}\coloneqq  \frac{1}{N}\sum_{n=1}^N    \operatorname{decompressor}\qty( \operatorname{compressor} \qty(\emti{k}{n}))$.
\end{lemma}
See its proof in Sec.~\ref{adasdas}.

\begin{lemma}\label{lem:sum_gk}
Suppose the setting in Theorem~\ref{thm:sgd} hold, when we use  the gradient 	$\tilde{\*g}_{k}  =   \frac{1}{N}\sum\nolimits_{n=1}^N \operatorname{decompressor}(\tilde{\*h}^n_{k+1})$   in Algorithm~\ref{alg:BitCom} to update the model weight $\bm{\theta}_k$ in SGD, see Eqn.~\eqref{eq:sgd-adam}, then we have:
\[
\EE\norm{ \sum_{i=0}^k \qty(\tilde{\*g}_i  - \*g_i) } \leq T_c \sqrt{d} \alpha c_\infty
 + \frac{\sqrt{d}k}{2s_e},
\]
where $\*g_k =  \frac{1}{N}\sum_{n=1}^N    \*g_k^n$.
\end{lemma}
See the proof in Sec.~\ref{sec:proof_sum_gk}.

\begin{lemma}\label{lem:sum_mk}
Suppose the setting in Theorem~\ref{thm:adam} hold, when we use  the gradient 	$\tilde{\*g}_{k}  =   \frac{1}{N}\sum\nolimits_{n=1}^N \operatorname{decompressor}(\tilde{\*h}^n_{k+1})$   in Algorithm~\ref{alg:BitCom} to update the model weight $\theta_k$ in Adam, see Eqn.~\eqref{eq:sgd-adam}.
Considering two sequences $\qty{\*m_k}_{k=1}^T$:
\[
\*m_k = (1-\beta_1)\*m_{k-1} + \beta_1 \*g_k,
\]
and $\qty{\tilde{{\*m}}_k}_{k=1}^T$,
\[
\tilde{\*m}_k = (1-\beta_1)\tilde{\*m}_{k-1} + \beta_1 \tilde{\*g}_k,
\]
where $ \*g_k =  \frac{1}{N}\sum_{n=1}^N \*g_k^n$.
Given the sequence of vectors $\qty{\bm{\eta}_k}_{k=0}^T$ such that element-wisely $\bm{\eta}_k$ satisfy: $\forall i \in [d]$,
\[
 \bm{\eta}_{k,i} \leq \eta c_u, \qquad \abs{\qty(\eta_{k,i} - \eta_{k-1,i})} \leq \beta_1 \eta c_u, \quad\forall k \in [T],
\]
then we have:
\[
\EE \norm{\sum_{i=0}^k \bm{\eta}_i \circ \qty(\tilde{\*m}_i  -  \*m_i) } \leq {(T_c+1 + \beta_1^2 k T_c) \sqrt{d} \alpha  c_\infty c_u \eta}  + \frac{k\sqrt{d}\eta c_u }{2s_e}.
\]
\end{lemma}
See the proof in Sec.~\ref{sec:proof_sum_mk}.
Lemma~\ref{lem:sum_gk} and Lemma~\ref{lem:sum_mk} demonstrate that, despite gradient compression, the deviation between the sum of compressed gradients and the sum of ground truth gradients remains bounded by the error of a single compression step. This ensures that compression errors do not accumulate indefinitely. However, to explicitly control the upper bound of the error introduced in a single iteration, the error reset mechanism is critical. It prevents the propagation of compression errors across iterations, maintaining controlled error scales. 

\begin{lemma}\label{lem:moving_mk}
Consider a moving average sequence:
\[
\*m_k = \qty(1-\beta)\*m_{k-1} + \beta {\*g}_k,
\]
where ${\*g}_{k} = \nabla f(\bm{\theta}_k) + \bm{\xi}_k$, and $\bm{\xi}_k$ is the sample noise such that $\EE \qty(\bm{\xi}_k) = \bm{0}$ and $\EE \norm{\bm{\xi}_k}^2 \leq \sigma^2$.
Then we have:
\[
\^E\qty(\norm{\*m_k - \nabla f(\bm{\theta}_k)}^2) \leq \qty(1-\beta)\^E\qty(\norm{\*m_{k-1} -\nabla f(\bm{\theta}_{k-1})}^2) + \frac{\qty(1-\beta)^2L^2}{\beta} \^E\qty(\norm{\bm{\theta}_{k-1} - \bm{\theta}_{k}}^2) + {\beta^2 \sigma^2}.
\]
\end{lemma}
See the proof in Sec.~\ref{sec:proof_moving_mk}.

\subsection{Convergence Guarantee of \method{}-integrated SGD}\label{sec:proof:SGD}
We provide the proof of Theorem~\ref{thm:sgd} in this section.
\begin{proof}
Consider two sequences $\{\tilde{\bm{\theta}}_k\}_{k=1}^T$, one is 
\[
\tilde{\bm{\theta}}_k = \tilde{\bm{\theta}}_{k-1} - \eta \tilde{\*g}_k = \bm{\theta}_{0} - \eta \sum_{i=0}^k \tilde{\*g}_i = \bm{\theta}_{0} - \eta \sum_{i=0}^k {\*g}_i  + \eta\qty(\sum_{i=0}^k \*g_i -\tilde{\*g}_i ),
\]
where $\tilde{\*g}_{k}  =   \frac{1}{N}\sum\nolimits_{n=1}^N \operatorname{decompressor}(\tilde{\*h}^n_{k+1})$   in Algorithm~\ref{alg:BitCom}.
Another sequence is $\qty{\bm{\theta}_k}_{k=1}^T$:
\[
\bm{\theta}_k = \bm{\theta}_{k-1} - \eta \*g_k = \bm{\theta}_{0} - \eta \sum_{i=0}^k \*g_i,
\]
where ${\*g}_{k}  \coloneqq   \frac{1}{N}\sum_{n=1}^N \*g_k^n = \nabla f(\tilde{\bm{\theta}}_k) + \bm{\xi}_k$, and $\bm{\xi}_k$ is the sample noise such that $\EE \qty(\bm{\xi}_k) = \bm{0}$ and $\EE \norm{\bm{\xi}_k}^2 \leq \sigma^2$. 
For and $k \in [T]$, we have:
\[
\EE \norm{\nabla f(\tilde{\bm{\theta}}_{k}) - \nabla f({\bm{\theta}}_{k}) } \leq L\EE \norm{\tilde{\bm{\theta}}_{k} - {\bm{\theta}}_{k} } = \eta L \EE \norm{\sum_{i=0}^{k} \qty(\tilde{\*g}_i  -  \*g_i) }\leq \eta L \qty(T_c \sqrt{d} \alpha c_\infty
 + \frac{\sqrt{d}k}{2s_e}) = \order{\eta + \frac{\eta k}{s_e}}.
\]
Since the function f is $L$-smooth, we can get:
\[
\begin{aligned}
\EE \qty(f(\bm{\theta}_{k+1})) & \leq \EE \qty(f(\bm{\theta}_{k})) + \EE \qty(\left\langle \nabla f\qty(\bm{\theta}_{k}), \bm{\theta}_{k+1}-\bm{\theta}_{k}\right\rangle)  + \frac{L}{2} \norm{\bm{\theta}_{k+1}-\bm{\theta}_{k}}^2 \\
& = \EE \qty(f(\bm{\theta}_{k})) - \eta \left\langle \nabla f\qty(\bm{\theta}_{k}), \nabla f\qty(\tilde{\bm{\theta}}_{k})\right\rangle  + \frac{L}{2} \norm{\bm{\theta}_{k+1}-\bm{\theta}_{k}}^2 \\
& \leq \EE \qty(f(\bm{\theta}_{k})) - \eta \norm{\nabla f\qty(\tilde{\bm{\theta}}_{k})}^2 + \eta \EE \left\langle  \nabla f(\tilde{\bm{\theta}}_{k})-  \nabla f\qty(\bm{\theta}_{k}), \nabla  f\qty(\tilde{\bm{\theta}}_{k})\right\rangle  + \frac{L}{2} \norm{\bm{\theta}_{k+1}-\bm{\theta}_{k}}^2 \\
& \leq \EE \qty(f(\bm{\theta}_{k})) - \eta \norm{\nabla f\qty(\tilde{\bm{\theta}}_{k})}^2 + \order{\eta \qty(\eta + \frac{\eta k}{s_e}) \sqrt{d}c_\infty}  + \frac{L\eta^2}{2} \qty(\norm{\nabla f\qty(\tilde{\bm{\theta}}_{k})}^2 + \frac{\sigma^2}{N}).
\end{aligned}
\]
By setting $\eta \leq \frac{1}{L}$, we have:
\[
   \frac{1}{T} \sum_{k=0}^T \EE \norm{\nabla f\qty(\tilde{\bm{\theta}}_{k})}^2  \leq \frac{2\qty(f(\tilde{\bm{\theta}}_{0}) - f(\tilde{\bm{\theta}}^*) )}{\eta T} + \order{{\eta \sqrt{d}c_\infty}} + \order{\frac{\eta T \sqrt{d} c_\infty}{s_e}} + \order{\frac{\eta L\sigma^2}{N}}.
\]
where $\tilde{\bm{\theta}}^* \in \argmin_{\bm{\theta}} f(\bm{\theta})$.
By setting $\eta=\order{\epsilon^2}$, $T = \order{\epsilon^{-4}}$, and $s_e = \Omega(\epsilon^{-4})$, we have:
\[
   \frac{1}{T} \sum_{k=0}^T \EE \norm{\nabla f\qty(\tilde{\bm{\theta}}_{k})}^2  \leq \order{\epsilon^2 \qty (\qty(f(\tilde{\bm{\theta}}_{0}) - f(\tilde{\bm{\theta}}^*) ) + \sqrt{d}c_\infty + \frac{L\sigma^2}{N})} = \order{\epsilon^2}.
\]
We finish the proof on SGD.
\end{proof}

\subsection{Convergence Guarantee of \method{}-integrated Adam-family Optimizers}\label{sec:proof:adam}
We provide the proof of Theorem~\ref{thm:adam} in this section.
\begin{proof}
Similar to the proof of Theorem~\ref{thm:sgd} in Sec.\ref{sec:proof:SGD}, we also consider two sequences $\qty{\bm{\theta}_k}_{k=1}^T$:
\[
\bm{\theta}_k = \bm{\theta}_{k-1} - \bm{\eta}_k \circ \*m_k = \bm{\theta}_{0} -  \sum_{i=0}^k \bm{\eta}_i \circ \*m_i, \quad \text{where} \quad  \*m_k = (1-\beta_1)\*m_{k-1} + \beta_1 \*g_k,
\]
and $\qty{\tilde{\bm{\theta}}_k}_{k=1}^T$,
\[
\tilde{\bm{\theta}}_k = \tilde{\bm{\theta}}_{k-1} - \bm{\eta}_k \circ \tilde{\*m}_k = \bm{\theta}_{0} - \sum_{i=0}^k \bm{\eta}_i \circ\tilde{\*m}_i = \bm{\theta}_{0} - \sum_{i=0}^k  \bm{\eta}_i \circ{\*m}_i  + \qty(\sum_{i=0}^k  \bm{\eta}_i \circ\qty(\*m_i - \tilde{\*m}_i ) ),
\]
where,
\[
\tilde{\*m}_k = (1-\beta_1)\tilde{\*m}_{k-1} + \beta_1 \tilde{\*g}_k.
\]
From Lemma~\ref{lem:moving_mk}, for and $k \in [T]$, we have:
\[
\EE \norm{\sum_{i=0}^k \bm{\eta}_i \circ \qty(\tilde{\*m}_i  -  \*m_i) } \leq {(T_c+1 + \beta_1^2 k T_c) \sqrt{d} \alpha  c_\infty c_u \eta}  + \frac{k\sqrt{d}\eta c_u }{2s_e} \leq  {2T_c\sqrt{d} \alpha  c_\infty c_u \eta}  + \frac{k\sqrt{d}\eta c_u }{2s_e},
\]
where the last inequality we use the assumption $\beta_1 = \order{\epsilon^2}$ and $k\leq T = \Omega(\epsilon^{-4})$. Then, we can get the following:
\begin{equation}\label{eq:adam:grad_bound}
 \EE \norm{\nabla f(\tilde{\bm{\theta}}_{k}) - \nabla f({\bm{\theta}}_{k}) } \leq L \EE \norm{\sum_{i=0}^k \bm{\eta}_i \circ \qty(\tilde{\*m}_i  -  \*m_i) } \leq {2LT_c \sqrt{d} \alpha  c_\infty c_u \eta}  + \frac{kL\sqrt{d}\eta c_u }{2s_e}.   
\end{equation}
We also get the following:
\begin{equation}\label{eq:adam:grad_bound_square}
 \EE \norm{\tilde{\bm{\theta}}_{k} - {\bm{\theta}}_{k} }^2 \leq \EE \norm{\sum_{i=0}^k \bm{\eta}_i \circ \qty(\tilde{\*m}_i  -  \*m_i) }^2 \leq {8T_c^2 {d} \alpha^2  c^2_\infty c^2_u \eta^2}  + \frac{k^2{d}c_u^2\eta^2 }{2s^2_e}.   
\end{equation}

For convenience, we let:
\[
c_a \coloneqq 8T_c^2 {d} \alpha^2  c^2_\infty c^2_u, \qquad
c_b \coloneqq \frac{{d}c_u^2}{2}, \qquad
c_e \coloneqq 2LT_c {d} \alpha  c^2_\infty c^2_u, \qquad
c_d \coloneqq \frac{L{d} c_\infty c^2_u }{2}.
\]

 Based on the $L$-smoothness of $f(\cdot)$, we have:
\begin{equation}\label{eq:adam:f_bound}
\begin{split}
& \EE \qty(f(\xmi{k+1})) \leq   \EE \qty( f(\xmi{k}) )+ \EE \langle \nabla f(\xmi{k}), \xmi{k+1} - \xmi{k} \rangle + \frac{L}{2} \EE \|\xmi{k+1} -\xmi{k}\|^2 \\ 
=  &  \EE \qty(f(\xmi{k})) -  \EE \left\langle \nabla f(\xmi{k}),  \bm{\eta}_k \circ{\mmi{k}} \right\rangle + \frac{L}{2} \EE \norm{\bm{\eta}_k \circ {\mmi{k}}}^2 \\ 
= & \EE \qty(f(\xmi{k})) -  \EE \left\langle \nabla f(\tilde{\bm{\theta}}_k),  \bm{\eta}_k \circ{\mmi{k}} \right\rangle + \frac{L}{2} \EE \norm{\bm{\eta}_k \circ {\mmi{k}}}^2 +
\EE \left\langle \nabla f(\tilde{\bm{\theta}}_k) - \nabla f(\xmi{k}),  \bm{\eta}_k \circ{\mmi{k}} \right\rangle
\\
\led{172} & \EE \qty(f(\xmi{k})) -  \EE \left\langle \nabla f(\tilde{\bm{\theta}}_k),  \bm{\eta}_k \circ{\mmi{k}} \right\rangle + \frac{L}{2} \EE \norm{\bm{\eta}_k \circ {\mmi{k}}}^2 + \qty(c_e{\eta}^2 + \frac{c_d\eta^2 k}{s_e})
\\
=  &  \EE \qty(f(\xmi{k})) + \frac{1}{2}   \EE \norm{ \sqrt{\bm{\eta}_k} \circ \qty( \nabla f(\tilde{\bm{\theta}}_k) - \mmi{k}) }^2  - \frac{1}{2}   \EE \left\|  \sqrt{\bm{\eta}_k} \circ  \nabla f(\tilde{\bm{\theta}}_k)    \right\|^2  - \frac{1}{2}   \EE \left\| \sqrt{\bm{\eta}_k} \circ   \mmi{k} \right\|^2\\
& + \frac{L}{2} \EE \norm{\bm{\eta}_k \circ {\mmi{k}}}^2 + \qty(c_e{\eta}^2 + \frac{c_d\eta^2 k}{s_e}) \\ 
\led{173} &  \EE \qty(f(\xmi{k})) + \frac{\eta c_u}{2}   \EE \left\|  \nabla  f(\tilde{\bm{\theta}}_k) - \mmi{k}  \right\|^2  - \frac{\eta c_l}{2}   \EE \left\|    f(\tilde{\bm{\theta}}_k)  \right\|^2  + \qty(\frac{L\eta^2 c^2_u}{2}- \frac{\eta c_l}{2})   \EE \left\|    \mmi{k} \right\|^2  + \qty(c_e{\eta}^2 + \frac{c_d\eta^2 k}{s_e}) \\ 
\led{174} &  \EE \qty(f(\xmi{k})) + \frac{\eta c_u}{2}   \EE \left\|  \nabla  f(\tilde{\bm{\theta}}_k) - \mmi{k}  \right\|^2  - \frac{\eta c_l}{2}   \EE \left\|    f(\tilde{\bm{\theta}}_k)  \right\|^2  - \frac{\eta c_l}{4} \EE \left\|    \mmi{k} \right\|^2  +\qty(c_e{\eta}^2 + \frac{c_d\eta^2 k}{s_e}) , 
\end{split}
\end{equation}  
 where \ding{172} is due to Eqn.~\eqref{eq:adam:grad_bound}, \ding{173} comes from the boundness of $\bm{\eta}_k$, and in \ding{174} we set $\eta \leq \frac{c_l }{2 L c_u^2}$ so that $ \frac{ L \eta^2 c_u^2 }{2} \leq  \frac{\eta c_l}{4}$. 
 
  Then from Lemma~\ref{lem:moving_mk}, we already have:
\begin{equation}\label{eq:adam:m_bound}
  \begin{split}
  	& \^E\qty(\norm{\*m_k - \nabla f(\tilde{\bm{\theta}}_k)})  \leq \qty(1-\beta_1)\^E\norm{\*m_{k-1} -\nabla f(\tilde{\bm{\theta}}_{k-1})}^2 + \frac{\qty(1-\beta_1)^2L^2}{\beta_1} \^E\norm{\tilde{\bm{\theta}}_{k-1} - \tilde{\bm{\theta}}_{k}}^2 + {\frac{\beta_1^2 \sigma^2}{N}}\\
  	\led{172}  & \qty(1-\beta_1)\^E\norm{\*m_{k-1} -\nabla f(\tilde{\bm{\theta}}_{k-1})}^2 + \frac{3\eta^2c_u^2\qty(1-\beta_1)^2L^2}{\beta_1} \^E\norm{\*m_{k-1}}^2 + {\frac{\beta_1^2 \sigma^2}{N}}  +3\qty(c_a{\eta}^2 + \frac{c_b\eta^2 k^2}{s^2_e}),
  \end{split}
  \end{equation}
  where \ding{172} holds since we have: $\xmi{k-1} - \xmi{k} = \bm{\eta}_{k-1} \circ {\mmi{k-1}}$,
  \[
\^E\norm{\tilde{\bm{\theta}}_{k-1} - \tilde{\bm{\theta}}_{k}}^2 \leq  
3 \qty(\^E\norm{{\bm{\theta}}_{k-1} - {\bm{\theta}}_{k}}^2 + \^E\norm{\tilde{\bm{\theta}}_{k} - {\bm{\theta}}_{k}}^2 + \^E\norm{{\bm{\theta}}_{k-1} - \tilde{\bm{\theta}}_{k-1}}^2).
  \]
  Next, we add Eqn.~\eqref{eq:adam:f_bound} and $a\times$ Eqn.~\eqref{eq:adam:m_bound}, and obtain:
\begin{equation*}
\begin{split}
     & \EE \qty(f(\xmi{k+1}))  + a \EE\|	\mmi{k+1} - \nabla f(\tilde{\bm{\theta}}_{k+1}) \|^2 \\
      \leq& \EE \qty(f(\xmi{k})) + \qty( \frac{\eta c_u}{2}   + a \qty(  1 - \beta_1   ) ) \EE \left\|  \nabla f(\tilde{\bm{\theta}}_k) - \mmi{k}  \right\|^2    - \frac{\eta c_l}{2}   \EE \left\|   \nabla f(\tilde{\bm{\theta}}_k)   \right\|^2 \\ 
      & - \left(\frac{\eta c_l}{4} -\frac{3a\eta^2c_u^2(1-\betai{1})^2L^2}{\betai{1}}\right)  \EE \left\|    \mmi{k} \right\|^2 + \frac{a\betai{1}^2  \sigma^2}{N} +  \qty(c_e{\eta}^2 + \frac{c_d\eta^2 k}{s_e} + 3a\qty(c_a{\eta}^2 + \frac{c_b\eta^2 k^2}{s^2_e})).
\end{split}
\end{equation*}  
 Let $a = \frac{\eta c_u}{\betai{1}}$ and $G(\xmi{k})  = \EE \qty(f(\xmi{k}))    + \frac{\eta c_u}{ \betai{1}} \EE \left\|  \nabla f(\tilde{\bm{\theta}}_k) - \mmi{k}  \right\|^2  $. Then we have:
   	\begin{equation*}
 	\begin{split}
 	G(\xmi{k+1}) 
 		\leq&      G(\xmi{k})  - \frac{\eta c_l}{2}   \EE \left\|   \nabla f(\tilde{\bm{\theta}}_k)   \right\|^2  - \left(\frac{\eta c_l}{4} -\frac{3\eta^3c_u^3(1-\betai{1})^2L^2}{\betai{1}^2}\right)  \EE \left\|    \mmi{k} \right\|^2 \\
   & +   \frac{  \eta \betai{1} c_u  \sigma^2}{N}   +  \qty(c_e{\eta}^2 + \frac{c_d\eta^2 k}{s_e} + \frac{3c_ac_u{\eta}^3}{\beta_1} + \frac{3c_uc_b\eta^3 k^2}{\beta_1s^2_e})      \\ 
 		 		\led{172}&      G(\xmi{k})  - \frac{\eta c_l}{2}   \EE \left\|   \nabla f(\tilde{\bm{\theta}}_k)   \right\|^2  - \frac{\eta c_l}{8}   \EE \left\|    \mmi{k} \right\|^2   \\
   & +   \frac{  \eta \betai{1} c_u  \sigma^2}{N}   +  \qty(c_e{\eta}^2 + \frac{c_d\eta^2 k}{s_e} + \frac{3c_ac_u{\eta}^3}{\beta_1} + \frac{3c_uc_b\eta^3 k^2}{\beta_1s^2_e})  ,
 	\end{split}
 \end{equation*}  
 where \ding{172} is due to the setting  $\eta \leq \frac{\betai{1} c_l^{0.5}}{5 c_u^{1.5} (1-\betai{1}) L}$. 
 
 Next, we can sum the above inequality from $k=0$ to $k=T-1$, and obtain:
\begin{equation*}
\begin{split}
\frac{1}{T} \sum_{k=0}^{T-1}  \EE \left[ \left\|   \nabla f(\tilde{\bm{\theta}}_k)   \right\|^2+ \frac{  1 }{4}    \left\|    \mmi{k} \right\|^2 	\right] 
\leq&     \frac{2 \qty(G(\xmi{0}) - G(\xmi{T-1}))}{c_l\eta T}   +    \frac{ 2 \betai{1} c_u  \sigma^2}{c_lN} +  \qty(\frac{2c_e{\eta}}{c_l} + \frac{2c_d\eta k}{c_l s_e} + \frac{6c_ac_u{\eta}^2}{c_l \beta_1} + \frac{6c_uc_b\eta^2 T^2}{c_l\beta_1s^2_e})         \\ 
    \led{172}&     \frac{2\Delta}{c_l \eta T}        +      \frac{2c_u  \sigma^2}{c_l \betai{1}N T}    +    \frac{ 2 \betai{1} c_u  \sigma^2}{c_lN} +  \qty(\frac{2c_e{\eta}}{c_l} + \frac{2c_d\eta k}{c_l s_e} + \frac{6c_ac_u{\eta}^2}{c_l \beta_1} + \frac{6c_uc_b\eta^2 T^2}{c_l\beta_1s^2_e}),
\end{split}
\end{equation*}  
 where \ding{172} comes from:
     	\begin{equation*} 
 	\begin{split}
G(\xmi{0}) - G(\xmi{T-1})  =  &  \EE \qty(f(\xmi{0}))    + \frac{\eta c_u}{\betai{1}} \EE \left\|  \nabla f(\xmi{0}) - \mmi{0}  \right\|^2  -    \EE \qty(f(\xmi{T-1}))   -  \frac{\eta c_u}{\betai{1}} \EE \left\|  \nabla f(\tilde{\bm{\theta}}_{T-1}) - \mmi{T-1}  \right\|^2 \\ 
\leq  &  \EE \qty(f(\xmi{0}))    + \frac{\eta c_u}{\betai{1}} \EE \left\|  \nabla f(\xmi{0}) - \mmi{0}  \right\|^2  -    \EE \qty(f(\xmi{T-1}))   \\
\leq  & \Delta    + \frac{\eta c_u \sigma^2}{ \betai{1}N},
 	\end{split}
 \end{equation*}  
 where $\Delta :=  \EE \qty(f(\xmi{0}))  -  \EE \qty(f(\xmi{*})) \geq  \EE \qty(f(\xmi{0}))  -  \EE \qty(f(\xmi{T-1}))$.

 By setting  $T = \Omega\qty(\epsilon^{-4})$, $\betai{1}= \order{\epsilon^2}$, $\eta = \order{\epsilon^2}$ and $s_e = \Omega(\epsilon^{-4})$, we have  
 	\begin{equation*}
 	\begin{split}
 		\frac{1}{T} \sum_{k=0}^{T-1}  \EE \left[ \left\|   \nabla f(\xmi{k})   \right\|^2+ \frac{  1 }{4}    \left\|    \mmi{k} \right\|^2 	\right] 
 		\leq \epsilon^2.  \\  
 	\end{split}
 \end{equation*} 
   The proof is completed. 
\end{proof}

\section{Proofs of Auxiliary Lemmas}\label{proofofAuxiliary}
Before providing the formal proofs for the auxiliary lemmas, we provide two foundation lemmas.
\begin{lemma}\label{lem:round}
    Ginve the bit length $p$ and the scalar $s>0$, consider the following operator: $\forall x>0$,
    \[
    \$C(x) \coloneqq \frac{\operatorname{float}\qty(\operatorname{round}_{\operatorname{p-bit}}\qty(x \times s))}{s}.
    \]
We have the following properties:
\[
\abs{\$C(x)} \leq \frac{2^{p+1}+1}{2s}, \qquad
\abs{x - \$C(x)} \leq \begin{cases}
     \frac{1}{2s}, \quad &\text{if } \abs{x} \leq \frac{2^p}{s}, \\ \\
     \abs{x} - \frac{2^p}{s}, \quad &\text{if} \abs{x}> \frac{2^p}{s},
 \end{cases} \qquad
 \abs{x - \$C(x)} \leq 2\abs{x}.
\]
\end{lemma}
\begin{proof}
    When $\abs{x}\leq \frac{2^p}{s}$, according to the properties of rational numbers, there are two integers $n_k$ and $n_{k+1}$ such that:
    \[
   \frac{n_{k}}{s} \leq x \leq \frac{n_{k+1}}{s}.
    \]
According to the definition of the operator $\$C(\cdot)$, we have:
\begin{equation}\label{eq:small_cx}
   \$C(x) = 
\begin{cases}
     \frac{n_{k}}{s}, \quad &\text{if } x \leq  \frac{2n_{k}+1}{2s}, \\
     \frac{n_{k+1}}{s}, \quad &\text{else} .
 \end{cases} 
\end{equation}

Hence,  when $\abs{x}\leq \frac{2^p}{s}$, we conclude that $\abs{x - \$C(x)} \leq \frac{1}{2s} \leq 2\abs{x}$. Considering the case  $\abs{x}> \frac{2^p}{s}$, we can easily find that $\$C(x) = \operatorname{sgn}(x)\frac{s^p}{s}$, where $\operatorname{sgn}$ is the sign function. Thus, we have $\abs{x-\$C(x)} = \abs{x} - \frac{2^p}{s}$. Then, we can conclude that:
\[
\abs{x - \$C(x)} \leq \begin{cases}
     \frac{1}{2s}, \quad &\text{if } \abs{x} \leq \frac{2^p}{s}, \\ 
     \abs{x} - \frac{2^p}{s}, \quad &\text{if} \abs{x}> \frac{2^p}{s},
 \end{cases}\qquad
 \abs{x - \$C(x)} \leq 2\abs{x}.
\]
On the other hand, we could verify that, $\abs{\$C(x)} = \frac{2^p}{s}$ when $\abs{x}> \frac{2^p}{s}$. By Eqn.~\eqref{eq:small_cx}, we have $\abs{\$C(x)} \leq  \abs{x}+ \frac{1}{2s}$ when  $\abs{x}\leq\frac{2^p}{s}$. Thus, we have:
\[
\abs{\$C(x)} \leq \frac{2^p}{s} + \frac{1}{2s} = \frac{2^{p+1}+1}{2s}.
\]
We finish the proof.
\end{proof}

\begin{lemma}\label{lem:err_bound}
We can get the bound of $\EE\norm{	\emti{k}{n} }$:
\[
\EE\norm{	\emti{k}{n} } \leq T_c\sqrt{d}\alpha  \beta c_\infty.
\]
We let $\emti{k}{}\coloneqq  \frac{1}{N}\sum_{n=1}^N   \emti{k}{n}$, which further yields:
		\begin{equation}\label{eq:e_bound}
		\begin{split}
	\EE \| \emti{k}{} \| = \EE \left\|\frac{1}{N}\sum_{n=1}^N    \emti{k}{n} \right\| \leq \frac{1}{N}\sum_{n=1}^N   \EE\| \emti{k}{n}\| \leq   T_c\sqrt{d}\alpha  \beta c_\infty.
\end{split}
\end{equation}
\end{lemma}
\begin{proof}
For convenience, we denote that:
\[
\$C_{c}\qty(\cdot) \coloneqq \operatorname{decompressor} \circ\operatorname{compressor}(\cdot, s, p), \quad \$C_{d}\qty(\cdot) \coloneqq \operatorname{decompressor} \circ\operatorname{compressor}(\cdot, s_e, p_e), \quad k^\prime \coloneqq k- \lfloor {k}/{T_c} \rfloor \times T_c.
\]
We prove the results by induction on the $i$-th element of the vector $\emti{k}{n}$. We first try to prove that
$\abs{\tilde{e}_{k,i}^n} \leq k^\prime \alpha \beta c_\infty$. Since we reset the error vector periodically, i.e., $\tilde{\*e}^n_k = \bm{0}$ when $k^\prime = 0$, only the case of $k^\prime=1$ to $k^\prime = T_c-1$ needs to be considered. Note that, for the $i$-th element of the vector $\emti{k}{n}$, we have:
\[
 \abs{\tilde{e}_{k,i}^n} \leq \qty(1-\beta) \abs{\tilde{e}_{k-1,i}^n} + \beta \abs{h_{k-1,i}^n - \$C_c(h_{k-1,i}^n)}.
\]
For $k^\prime=1$, i.e., just after the resetting, we consider two cases. If $\abs{h_{k-1,i}^n} \leq \frac{2^p}{s}$, then:
\[
 \abs{\tilde{e}_{k,i}^n} \leq \frac{\beta}{2s}.
\]
If  $\abs{h_{k-1,i}^n}> \frac{2^p}{s}$, we have:
\[
\begin{aligned}
 \abs{\tilde{e}_{k,i}^n} & \leq \beta \abs{h_{k-1,i}^n - \$C_c(h_{k-1,i}^n)} \leq  \beta \qty(\abs{h_{k-1,i}^n}-\frac{2^p}{s})\leq  \beta \qty(\abs{g_{k-1,i}^n}-\frac{2^p}{s}) \leq  \alpha \beta c_\infty,
\end{aligned}
\]
where the last inequality comes from the assumption $\frac{2^p}{s(1-\alpha)}\geq c_\infty$. Hence, we can conclude that $\abs{\tilde{e}_{k,i}^n} \leq k^\prime \alpha \beta c_\infty$ for $k^\prime=1$, i.e., the first error vector during this reset period holds the bounds.
Now, we assume that for all $\qty(\lfloor {k}/{T_c} \rfloor \times T_c) \leq t < k$ hold the bounds $\abs{\tilde{e}_{t,i}^n} \leq t^\prime \alpha \beta c_\infty$. Then we consider the case for $\tilde{e}_{k,i}^n$.
If $\abs{h_{k-1,i}^n} \leq \frac{2^p}{s}$, then:
\[
 \abs{\tilde{e}_{k,i}^n} \leq \qty(1-\beta) \abs{\tilde{e}_{k-1,i}^n} +  \frac{\beta}{2s}.
\]
If  $\abs{h_{k-1,i}^n}> \frac{2^p}{s}$, we have:
\[
\begin{aligned}
 \abs{\tilde{e}_{k,i}^n} & \leq \qty(1-\beta) \abs{\tilde{e}_{k-1,i}^n} + \beta \abs{h_{k-1,i}^n - \$C_c(h_{k-1,i}^n)} \leq \qty(1-\beta) \abs{\tilde{e}_{k-1,i}^n} +  \beta \qty(\abs{h_{k-1,i}^n}-\frac{2^p}{s}) \\
 & \leq \qty(1-\beta) \abs{\tilde{e}_{k-1,i}^n} + \beta \qty(\abs{g_{k-1,i}^n} + \abs{\tilde{e}_{k-1,i}^n} + \abs{\tilde{e}_{k-1,i}^n - \$C_{d}(\tilde{e}_{k-1,i}^n)}-\frac{2^p}{s}) \\
& \leq \abs{\tilde{e}_{k-1,i}^n} + \alpha \beta c_\infty.
\end{aligned}
\]
where the last inequality comes from the assumption $\frac{2^p}{s}\geq (1-\alpha)c_\infty + \frac{1}{2s_e}$ and $T_c\alpha  \beta  s_e c_\infty \leq  2^{p_e}$. Combing all the cases together, we get:
\[
 \abs{\tilde{e}_{k,i}^n} \leq  \abs{\tilde{e}_{k-1,i}^n} + \alpha \beta c_\infty
  \leq \abs{\tilde{e}_{k-2,i}^n} + 2\alpha \beta c_\infty
  \leq  k^\prime \alpha  \beta c_\infty.
\]
Then, the conclusion is obvious since $k^\prime \leq T_c$.
\end{proof}

\subsection{Proof of Lemma~\ref{gradientrlation}}\label{adasdas}
\begin{proof}
	For the gradient $\gmbi{k}$, it is defined as:
\begin{equation*}
\begin{split}
    \gmbi{k} = & \frac{1}{N}\sum\nolimits_{n=1}^N \operatorname{decompressor}(\tilde{\*h}^n_{k+1})
    =  \frac{1}{N}\sum_{n=1}^N \operatorname{decompressor}(\operatorname{compressor}(\hmi{k+1}{n})) \\
    = & \frac{1}{N}\sum_{n=1}^N \*h_{k+1}^n +  \operatorname{decompressor}(\operatorname{compressor}(\hmi{k+1}{n})) - \hmi{k+1}{n} \\
    \lee{172} & \frac{1}{N}\sum_{n=1}^N  \gmi{k}{n} +  \operatorname{decompressor}(\emi{k}{n})  +  \operatorname{decompressor}(\operatorname{decompressor}(\hmi{k+1}{n})) - \hmi{k+1}{n} \\
    \lee{173} & \frac{1}{N}\sum_{n=1}^N  \gmi{k}{n} +  \operatorname{decompressor}(\emi{k+1}{n}) - \deltami{k+1}{n},
\end{split}
\end{equation*}
	where  \ding{172} uses $ \*h_{k+1}^n= \*g_k^n +  \operatorname{decompressor}(\emi{k}{n})$, and in \ding{173}, we define:
	\begin{equation*}
		\begin{split}
			\deltami{k}{n} :=\hmi{k}{n}  -  \operatorname{decompressor}(\operatorname{compressor}(\hmi{k}{n})).
		\end{split}
	\end{equation*}
	
	At the same time, we have 
	\begin{equation*}
		\begin{split}
			\operatorname{decompressor}(\emi{k}{n}) - \deltami{k+1}{n} \lee{172} &  \operatorname{decompressor}(\emi{k}{n}) -  \frac{1}{\beta} \left( \emti{k+1}{n} - (1-\beta) \emti{k}{n} \right)   \\
			= &  \operatorname{decompressor}( \operatorname{compressor} \qty(\emti{k}{n}))  -  \frac{1}{\beta} \left( \emti{k+1}{n} - (1-\beta) \emti{k}{n} \right)   \\
			= &  \frac{1}{\beta} \left( \emti{k}{n} - \emti{k+1}{n} \right)  +   \operatorname{decompressor}( \operatorname{compressor} (\emti{k}{n}))  -    \emti{k}{n},
		\end{split}
	\end{equation*}
	where in \ding{172}  we use $\deltami{k}{n} = \frac{1}{\beta} \left( \emti{k}{n} - (1-\beta) \emti{k-1}{n} \right)$ from the Algorithm. Accordingly, we have: 
	\begin{equation*}
		\begin{split}
			\gmbi{k} =    \frac{1}{N}\sum_{n=1}^N  \left( \gmi{k}{n} +  \frac{ 1}{\beta} \left( \emti{k}{n} - \emti{k+1}{n} \right)  +   \operatorname{decompressor}( \operatorname{compressor} (\emti{k}{n}))  -    \emti{k}{n}  \right).
		\end{split}
	\end{equation*}
	We complete the proof. 
\end{proof}

\subsection{Proof of Lemma~\ref{lem:sum_gk}}\label{sec:proof_sum_gk}
\begin{proof}
In Lemma~\ref{gradientrlation}, we already have:
$
\gmbi{k}  =  \gmi{k}{} +  \frac{1}{\beta} \left( \emti{k}{} - \emti{k+1}{} \right)  +   \emhi{k}{}  -    \emti{k}{}
$,
where we let $ \gmi{k}{} \coloneqq \frac{1}{N}\sum_{n=1}^N    \gmi{k}{n}$,  $\emti{k}{}\coloneqq  \frac{1}{N}\sum_{n=1}^N   \emti{k}{n}$,  and $\emhi{k}{}\coloneqq  \frac{1}{N}\sum_{n=1}^N    \$C_d \qty(\emti{k}{n})$.
In this way, by the results from Lemma~\ref{lem:err_bound}, we can bound:
\[
\begin{split}
\EE\norm{ \sum_{i=0}^k \qty(\gmbi{i}  -  \gmi{i}{}) } &=  \EE \norm{  \frac{1}{\beta} \sum_{i=0}^k\qty(\qty( \emti{i}{} - \emti{i+1}{} )  +   \emhi{i}{}  -    \emti{i}{}) } 
\leq  \frac{1}{\beta} \EE \norm{\emti{k}{}}
 + \frac{1}{N} \EE \norm{\sum_{n=1, i=1}^{N,k}\emti{i}{n} - \$C_d \qty(\emti{i}{n})}\\
 &\led{172} \frac{T_c \sqrt{d} \alpha \beta c_\infty}{\beta} 
 + \frac{1}{N}\frac{\sqrt{d}Nk}{2s_e}
 =T_c \sqrt{d} \alpha c_\infty
 + \frac{\sqrt{d}k}{2s_e},
\end{split}
\]
where \ding{172} comes from Eqn.~\eqref{eq:e_bound}, and the assumption that $\frac{2^{p_e}}{s_e}\geq T_c\alpha \beta c_\infty$. We finish the proof.
\end{proof}
\subsection{Proof of Lemma~\ref{lem:sum_mk}}\label{sec:proof_sum_mk}
\begin{proof}
    We first expand the formula of $\*m_k$:
\[
\*m_k = (1-\beta_1)\*m_{k-1} + \beta_1 \*g_k = (1-\beta_1)^2\*m_{k-2} + \beta_1 \qty(\*g_k + \qty(1-\beta_1) \*g_{k-1}) = \qty(1-\beta_1)^{k-1} \*g_0 + \beta_1 \sum_{t=1}^k \qty(1-\beta_1)^{k-t} \*g_t.
\]
Similarly, we also have:
\[
\tilde{\*m}_k = \qty(1-\beta_1)^{k-1} \*g_0 + \beta_1 \sum_{t=1}^k \qty(1-\beta_1)^{k-t} \tilde{\*g}_t.
\]
Then, by Lemma~\ref{gradientrlation}, we have:
\[
\gmbi{k}  =  \gmi{k}{} +  \frac{1}{\beta} \left( \emti{k}{} - \emti{k+1}{} \right)  +   \emhi{k}{}  -    \emti{k}{},
\]
where we let $ \gmi{k}{} \coloneqq \frac{1}{N}\sum_{n=1}^N    \gmi{k}{n}$,  $\emti{k}{}\coloneqq  \frac{1}{N}\sum_{n=1}^N   \emti{k}{n}$,  $\emhi{k}{}\coloneqq  \frac{1}{N}\sum_{n=1}^N    \$C_d \qty(\emti{k}{n})$.
Then, we can get:
\[
\sum_{i=1}^k \bm{\eta}_i \circ \qty(\tilde{\*m}_i  -  \*m_i) = \frac{\beta_1}{\beta} \sum_{i=1}^k  \bm{\eta}_i \circ \qty(\sum_{t=1}^i \qty(1-\beta_1)^{i-t} \qty(\tilde{\*e}_t-\tilde{\*e}_{t+1} + \beta \bm{\Delta}_{t})),
\]
where $\bm{\Delta}_t \coloneqq \emhi{i}{}  -    \emti{i}{}$.
Now, we element-wisely analyze the right side term to get its upper bound. For each element of the right side term (we omit the index of dimension to simplify the notation), we have:
\[
\begin{aligned}
 & \abs{\frac{\beta_1}{\beta} \sum_{i=1}^k  {\eta}_i \qty(\sum_{t=1}^i \qty(1-\beta_1)^{i-t} \qty(\tilde{e}_t-\tilde{e}_{t+1} + \beta {\Delta}_{t}))} \\
  = & \abs{\frac{\beta_1}{\beta} \sum_{i=1}^k  {\eta}_i\sum_{t=1}^i \qty(1-\beta_1)^{i-t} \qty(\tilde{e}_t-\tilde{e}_{t+1}) + \beta_1 \sum_{i=1}^k \sum_{t=1}^i  {\eta}_i \qty(1-\beta_1)^{i-t} {\Delta}_{t}}\\
  \led{172} &  \abs{\frac{\beta_1}{\beta} \sum_{i=1}^k  {\eta}_i\sum_{t=1}^i \qty(1-\beta_1)^{i-t} \qty(\tilde{e}_t-\tilde{e}_{t+1})} + \frac{\beta_1\eta c_u }{2s_e} \sum_{i=1}^k\sum_{t=1}^i \qty(1-\beta_1)^{i-t} \\
  =& \abs{\frac{\beta_1}{\beta} \sum_{i=1}^k  {\eta}_i\sum_{t=1}^i \qty(1-\beta_1)^{i-t} \qty(\tilde{e}_t-\tilde{e}_{t+1})} + \frac{\eta c_u }{2s_e} \qty(k-\sum_{i=1}^k\qty(1-\beta_1)^i)\\
  \leq & \abs{\frac{\beta_1}{\beta} \sum_{i=1}^k  {\eta}_i\sum_{t=1}^i \qty(1-\beta_1)^{i-t} \qty(\tilde{e}_t-\tilde{e}_{t+1})} + \frac{k\eta c_u }{2s_e},
\end{aligned}
\]
where in \ding{172} we use the bound $\abs{\tilde{e}_{k,i}^n} \leq  T_c\alpha  \beta c_\infty$ from the proof of Lemma~\ref{lem:sum_gk}, results from Lemma~\ref{lem:round} and the assumption that $\frac{2^{p_e}}{s_e}\geq T_c\alpha \beta c_\infty$.
Note that we have the following:
\[
\begin{aligned}
& \abs{\sum_{i=1}^k  {\eta}_i\sum_{t=1}^i \qty(1-\beta_1)^{i-t} \qty(\tilde{e}_t-\tilde{e}_{t+1})} \\
= & \abs{\sum_{i=1}^k\qty(\eta_i \qty(1-\beta_1)^{i-1})\tilde{e}_1 + 
\sum_{t=2}^k \qty(\sum_{i=t}^{k} \qty(\eta_i - \eta_{i-1})\qty(1-\beta_1)^{i-t})\tilde{e}_t  - \sum_{t=2}^k \qty(\eta_{k} \qty(1-\beta_1)^{k-t+1})\tilde{e}_t}  \\
\led{172} & c_u \eta \qty(T_c+1) \alpha \beta c_\infty  \sum_{i=1}^k\qty(\qty(1-\beta_1)^{i-1}) + T_c\alpha  \beta c_\infty \abs{\sum_{t=2}^k \qty(\sum_{i=t}^{k} \qty(\eta_i - \eta_{i-1})\qty(1-\beta_1)^{i-t})} \\
\led{173} & c_u \eta \qty(T_c+1) \alpha c_\infty \frac{\beta}{\beta_1} +  c_u \eta T_c\alpha  \beta c_\infty \beta_1 \qty(k-1),
  \end{aligned}
\]
where in \ding{172} we use the bound from Lemma~\ref{lem:err_bound}, i.e., $e_1 \leq \alpha \beta c_\infty$ and $e_t \leq T_c\alpha \beta c_\infty$ for $t>2$. In \ding{173}, we use the assumption $\abs{\qty(\eta_i - \eta_{i-1})} \leq \beta_1 c_u \eta, \forall i \in [T]$.

Then, combining all bounds together, we can get the following:
\[
\begin{aligned}
 & \abs{\frac{\beta_1}{\beta} \sum_{i=1}^k  {\eta}_i \qty(\sum_{t=1}^i \qty(1-\beta_1)^{i-t} \qty(\tilde{e}_t-\tilde{e}_{t+1} + \beta {\Delta}_{t}))} \\
  \leq & \abs{\frac{\beta_1}{\beta} \sum_{i=1}^k  {\eta}_i\sum_{t=1}^i \qty(1-\beta_1)^{i-t} \qty(\tilde{e}_t-\tilde{e}_{t+1})} + \frac{k\eta c_u }{2s_e}\\
  \leq & \frac{\beta_1}{\beta} \qty(c_u \eta \qty(T_c+1) \alpha c_\infty \frac{\beta}{\beta_1} +  c_u \eta T_c\alpha  \beta c_\infty \beta_1 \qty(k-1)) + \frac{k\eta c_u }{2s_e}\\
  \leq & \qty(T_c+1 + \beta_1^2 k T_c) \alpha c_\infty  c_u \eta + \frac{k\eta c_u }{2s_e}.
\end{aligned}
\]
Then, we can conclude that:
\[
\EE \norm{\sum_{i=0}^k \bm{\eta}_i \circ \qty(\tilde{\*m}_i  -  \*m_i) } \leq {(T_c+1 + \beta_1^2 k T_c) \sqrt{d} \alpha  c_\infty c_u \eta}  + \frac{k\sqrt{d}\eta c_u }{2s_e}.
\]
We finish the proof.
\end{proof}
\subsection{Proof of Lemma~\ref{lem:moving_mk}}\label{sec:proof_moving_mk}
\begin{proof}
we denote ${\*g}^{full}_{k}\coloneqq \nabla f(\bm{\theta}_k)$ for convenience.
Note that we have:
\[
\begin{aligned}
  \*m_k - \*g^{full}_k = &  \qty(1-\beta)\qty(\*m_{k-1} - \*g^{full}_{k-1}) + (1-\beta)\*g^{full}_{k-1} - \*g^{full}_k +   \beta {\*g}_k\\
  =& \qty(1-\beta)\qty(\*m_{k-1} - \*g^{full}_{k-1}) + (1-\beta)\qty(\*g^{full}_{k-1} - \*g^{full}_k) + \beta \qty( {\*g}_k - \*g^{full}_k).
\end{aligned}
\]
Then, take expectation on both sides:
\[
\begin{aligned}
 & \^E\qty(\norm{\*m_k - \*g^{full}_k}^2)\\
= & \qty(1-\beta)^2\^E\qty(\norm{\*m_{k-1} - \*g^{full}_{k-1}}^2)+\qty(1-\beta)^2\^E\qty(\norm{\*g^{full}_{k-1} - \*g^{full}_{k}}^2) + {\beta^2 \sigma^2}+\\
&  2 \qty(1-\beta)^2\^E\qty(\innerprod{\*m_{k-1} - \*g^{full}_{k-1}, \*g^{full}_{k-1} - \*g^{full}_{k}}) \\
\leq & \qty(\qty(1-\beta)^2 + \qty(1-\beta)^2 a)\^E\qty(\norm{\*m_{k-1} - \*g^{full}_{k-1}}^2) + \qty( 1+\frac{1}{a})\qty(1-\beta)^2\^E\qty(\norm{\*g^{full}_{k-1} - \*g^{full}_{k}}^2) + {\beta^2 \sigma^2}\\
\led{172} & \qty(1-\beta)\^E\qty(\norm{\*m_{k-1} - \*g^{full}_{k-1}}^2) + \frac{\qty(1-\beta)^2}{\beta} \^E\qty(\norm{\*g^{full}_{k-1} - \*g^{full}_{k}}^2) + {\beta^2 \sigma^2} \\ 
\leq & \qty(1-\beta)\^E\qty(\norm{\*m_{k-1} - \*g^{full}_{k-1}}^2) + \frac{\qty(1-\beta)^2L^2}{\beta} \^E\qty(\norm{\bm{\theta}_{k-1} - \bm{\theta}_{k}}^2) + {\beta^2 \sigma^2},
\end{aligned}
\]
where for~\ding{172}, we set $a = \frac{\beta}{1-\beta}$. We finish the proof.
\end{proof}


\end{document}